\documentclass[sigconf]{acmart}

\usepackage{multirow}
\usepackage{subfigure}
\usepackage{algorithm}
\usepackage{algorithmic}

\usepackage{bm}
\usepackage{amsmath}
\usepackage{enumitem}

\newtheorem{theorem}{Theorem}
\newtheorem{Definition}{Definition}

\usepackage{pdftexcmds}
\usepackage{catchfile}
\usepackage{ifluatex}
\usepackage{ifplatform}
\usepackage{caption}
\usepackage{color}
\usepackage{nicefrac}
\usepackage{caption}
\usepackage{booktabs}

\AtBeginDocument{%
\providecommand\BibTeX{{%
		\normalfont B\kern-0.5em{\scshape i\kern-0.25em b}\kern-0.8em\TeX}}}
	
\usepackage{array}
\newcolumntype{L}[1]{>{\raggedright\let\newline\\\arraybackslash\hspace{0pt}}m{#1}}
\newcolumntype{C}[1]{>{\centering\let\newline  \\\arraybackslash\hspace{0pt}}m{#1}}
\newcolumntype{R}[1]{>{\raggedleft\let\newline \\\arraybackslash\hspace{0pt}}m{#1}}

\renewcommand\footnotetextcopyrightpermission[1]{} 
\begin{document}

\title{Searching to Sparsify Tensor Decomposition for \\ N-ary Relational Data}

%
\author{Shimin DI
}
\affiliation{%
  \institution{The Hong Kong University of Science and Technology}
  \city{Hong Kong SAR}
  \country{China}}
\email{sdiaa@cse.ust.hk}

\author{Quanming YAO$^*$}
\affiliation{
	\institution{4Paradigm Inc. \\
		EE, Tsinghua University}
	\city{Beijing, China}
	}
\email{qyaoaa@connect.ust.hk}

\author{Lei CHEN
}
\affiliation{%
	\institution{The Hong Kong University of Science and Technology}
	\city{Hong Kong SAR}
	\country{China}}
\email{leichen@cse.ust.hk}

%
%
%
%
%


\begin{abstract}
Tensor, an extension of the vector and matrix to the multi-dimensional case, is a natural way to describe the N-ary relational data. Recently, tensor decomposition methods have been introduced into N-ary relational data and become state-of-the-art on embedding learning. However, the performance of existing tensor decomposition methods is not as good as desired. First, they suffer from the data-sparsity issue since they can only learn from the N-ary relational data with a specific arity, i.e., parts of common N-ary relational data. Besides, they are neither effective nor efficient enough to be trained due to the over-parameterization problem. In this paper, we propose a novel method, i.e., S2S, for effectively and efficiently learning from the N-ary relational data. Specifically, we propose a new tensor decomposition framework, which allows embedding sharing to learn from facts with mixed arity. 
Since the core tensors may still suffer from the over-parameterization,
we propose to reduce parameters by sparsifying the core tensors while retaining their expressive power using neural architecture search (NAS) techniques, which can search for data-dependent architectures.
As a result, the proposed S2S not only guarantees to be expressive but also efficiently learns from mixed arity. Finally, empirical results have demonstrated that S2S is efficient to train and achieves state-of-the-art performance.
\footnote{The work is done when S. Di was an intern at 4Paradigm Inc and mentored by Q. Yao;
	and Q. Yao is the correspondence author.}
\end{abstract}

%



\maketitle

\section{Introduction}
\label{sec:intro}

As an important way to explore and organize human knowledge, web-scale knowledge bases (KBs, i.e., N-ary relational data) \cite{suchanek2007yago,auer2007dbpedia,bollacker2008freebase} has promoted a series of web applications, e.g., semantic search~\cite{xiong2017explicit}, question answering~\cite{lukovnikov2017neural}, and recommendation system~\cite{zhang2016collaborative,cao2019unifying}.
Generally, 
the N-ary relational data contains n-ary facts, that is formed by $n$ entities with a relation $r$ such as $(r, e_1,\cdots,e_n)$ (i.e., arity is $n$).
For example, \textit{playedCharacterIn} is one of common 3-ary relations, involved with an actor, a character, and a movie in a 3-ary fact (\textit{playedCharacterIn, LeonardNimoy, Spock, StarTrek 1}).
Given a fact, the link prediction task is one of the crucial tasks in the N-ary relational data, which is to verify whether a fact is plausible or not.
Previous studies~\cite{bordes2013translating,yang2014embedding,dettmers2018convolutional,kazemi2018simple,zhang2019autosf} focus on handling the link prediction task on a special case of the N-ary relational data, knowledge graphs (KGs, i.e., binary relational data) \cite{nickel2015review,wang2017knowledge}.
Recently, how to handle the general N-ary relational data has attracted lots of attention~\cite{wen2016representation,zhang2018scalable,guan2019link,rosso2020beyond,fatemi2019knowledge,guan2020neuinfer}.
Firstly, it is essential to handle hyper-relational facts
(i.e., n-ary facts with $n>2$)
because they are very common in KBs.
It has been reported in \cite{wen2016representation} that more than 30\% of the entities in Freebase \cite{bollacker2008freebase} involves in the hyper-relational facts.
Moreover, the facts with high-arity may provide benefits in the question answering scenario \cite{ernst2018highlife} since it usually contains more complete information compared with binary facts.


Many models have been proposed to tackle the link prediction task on the N-ary relational data.
The translational distance models m-TransH \cite{wen2016representation} and RAE \cite{zhang2018scalable} extend a well-known method TransH~\cite{wang2014knowledge} from binary to the n-ary scenario.
But TransH cannot handle certain relations~\cite{kazemi2018simple,sun2019rotate}.
Thus, it is regarded as inexpressive
since a fully expressive model should be able to handle arbitrary relation patterns on the binary case~\cite{kazemi2018simple}.
Consequently, m-TransH and RAE are also not expressive.
However, the expressive ability largely determines the performance of embedding models.
Thus, the expressiveness of translational distance models worsens their performance in the case of N-ary relational data.
Furthermore, the neural network models, NaLP~\cite{guan2019link}, HINGE~\cite{rosso2020beyond}, and NeuInfer~\cite{guan2020neuinfer}, achieve good performance by employing complex neural networks to learn embeddings.
But they all introduce an enormous amount of parameters, which contradicts the linear time and space requirement in knowledge bases \cite{bordes2013irreflexive}.

\begin{table*}[t]
	\centering
	\caption{Summary of existing n-ary works.
		Whether a scoring function is expressive depends on its capability of handling common relation patterns as in \cite{zhang2019autosf}.
		The Mixed-arity indicates whether a model jointly learn from the N-ary relational data with mixed arity.
		$N$ is the maximum arity of facts.
		$n_e$ and $n_r$ are the number of entities and relations, respectively. 
		$d_e$ and $d_r$ denote the dimensionality of embeddings on entity and relation, respectively.
		And $d_{\max}=\max_{i} d_i$ with $\prod_{i=1}^cd_i=d_e^nd_r$ in GETD \cite{liu2020generalizing}. 
		The time is the computational cost of calculating the score of the single n-ary fact based on $d=d_e=d_r$.}
	\label{table:summary}
	\setlength\tabcolsep{10pt}
	\begin{tabular}{ c | c | c | c | c | c }
		\toprule
		\multirow{2}{*}{\bf Type} & \multirow{2}{*}{\bf Models}           & \multicolumn{2}{c|}{\bf Effectiveness} &      \multicolumn{2}{c}{\bf Efficiency}       \\ \cmidrule{3-6}
		                          &                                       & Expressive & Mixed-arity               & Time         & Space                          \\ \midrule
		Translational Models      & m-TransH \cite{wen2016representation} & $\times$   & $\checkmark$                   & $O(d)$       & $O(n_ed_e+n_rd_r)$             \\ \cmidrule{2-6}
		                          & RAE \cite{zhang2018scalable}          & $\times$   & $\checkmark$                   & $O(d^2)$     & $O(n_ed_e+n_rd_r)$             \\ \midrule
		Neural Network Models     & NaLP \cite{guan2019link}              & unknown        & $\checkmark$                   & $O(d^2)$     & $O(n_ed_e+Nn_rd_r)$            \\ \cmidrule{2-6}
		                          & HINGE \cite{rosso2020beyond}          & unknown        & $\checkmark$                   & $O(d^2)$     & $O(n_ed_e+Nn_rd_r)$            \\ \cmidrule{2-6}
		                          & NeuInfer \cite{guan2020neuinfer}      & unknown        & $\checkmark$                   & $O(d^2)$     & $O(n_ed_e+Nn_rd)$              \\ \midrule
		Tensor Decomposition      & n-TuckER \cite{liu2020generalizing}   & $\checkmark$    & $\times$                  & $O(d^{n+1})$ & $O(n_ed_e+n_rd_r+d_e^nd_r)$    \\ \cmidrule{2-6}
		Models                    & GETD \cite{liu2020generalizing}       & $\checkmark$    & $\times$                  & $O(d^3)$    & $O(n_ed_e+n_rd_r+cd_{\max}^3)$ \\ \cmidrule{2-6}
		                          & S2S                                   & $\checkmark$    & $\checkmark$                   & $O(d)$      & $O(n_ed_e+n_rd_r)$           \\ \bottomrule
	\end{tabular}
\end{table*}

Tensor decomposition models \cite{balazevic2019tucker,liu2020generalizing} introduce a natural way to model N-ary relational data with a $(n+1)$-order tensor and become state-of-the-art because of their expressiveness.
TuckER \cite{balazevic2019tucker} proposes to model the binary relational data with a 3-order tensor and then decomposes it for embedding learning.
It is easy to extend TuckER from binary to high-arity relational data by modeling n-ary facts with a high-order tensor, 
named n-TuckER \cite{balazevic2019tucker,liu2020generalizing}.
However, such a simple extension will lead to the curse of dimensionality due to the large size of the core tensor.
Therefore, GETD \cite{liu2020generalizing}
simplifies the core tensor  
with Tensor Ring Decomposition~\cite{zhao2016tensor} to reduce the model complexity.
Then, 
GETD achieves outstanding performance in the N-ary relational data because of less model complexity and expressive guarantee  \cite{liu2020generalizing}.

However, existing tensor decomposition models for N-ary relational data still suffer from two issues: 
\textit{data sparsity} and \textit{over-parameterization}.
First, it is well-known that the N-ary relational data is very sparse, which is difficult for training and learning \cite{pujara2017sparsity}.
But existing tensor decomposition models \cite{balazevic2019tucker,liu2020generalizing} can only learn embeddings from facts with a specific arity $n$, while the N-ary relational data usually contains facts with different arities \cite{wen2016representation,rosso2020beyond}.
In other words, tensor decomposition models cannot leverage all known facts of the given N-ary relational data, which causes the data sparsity issue to become even more severe.
Second, current tensor decomposition models achieve the expressive capability by maintaining an over-parameterized core tensor, even GTED requires cubic model complexity.
Such over-parameterization for expressiveness not only makes the model inefficient but also difficult to train.
We summarize the above existing models for N-ary relational data in Table~\ref{table:summary}. 
We first compare the two main factors that affect the effectiveness of current models, 
the expressive capability, and whether the model can learn from facts with mixed arity.
Then, 
to demonstrate whether the model requires a large number of parameters, we compare their efficiency from the infer time and size of parameter space.
Obviously, none of the existing works can cover all the aspects.

This paper aims to alleviate the data sparsity and over-parameterization issues of 
existing tensor decomposition models for n-ary relation data learning.
To handle the data sparsity issue, we propose to partially share embeddings across arities and jointly learn embeddings from the N-ary relational data with mixed arity.
Then, 
motivated by the structurally sparse patterns discovered from existing tensor models on binary relational data
and the success of neural architecture search (NAS)~\cite{automl_book,yao2018taking} on designing data-specific deep networks,
we search to sparsify the dense core tensors using NAS techniques 
to avoid over-parameterization.
In this way,
we address the issues of data sparsity and over-parameterization while retaining the expressiveness of tensor models.

We summarize the important notations in Table~\ref{table:notation},
and our contributions are listed as follows:
\begin{itemize}[leftmargin=*]
	\item 
	We propose a new model, i.e., S2S, to learn from N-ary relational data,
	which simultaneously addresses the data-sparsity and over-parameterization issue faced by existing tensor decomposition models.
	
	\item To capture the data-specific knowledge,
	we propose a novel approach to search for multiple sparse core tensors, which are utilized to jointly learn from any given N-ary relational data with mixed arity.
	
	\item We test the proposed model on the link prediction task in both binary and N-ary relational data.
	Experimental results show that S2S not only achieves outstanding performance in embedding learning but also improves efficiency.
\end{itemize}


\begin{figure*}[t]
	\centering
	\subfigure[TuckER.]{\includegraphics[width=0.47\columnwidth]{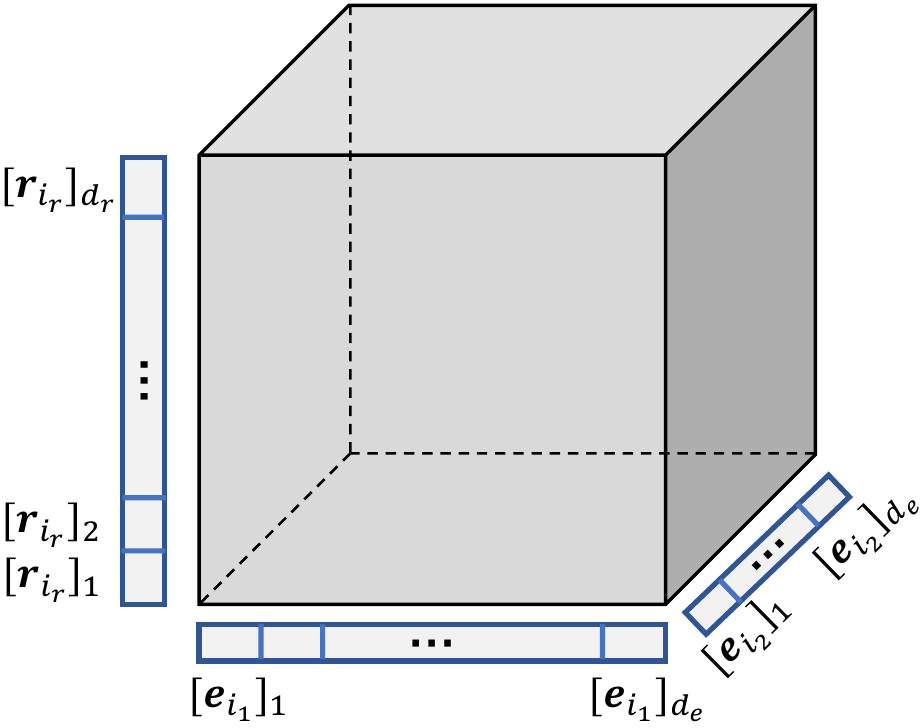}}
	\subfigure[DistMult.]{\includegraphics[width=0.47\columnwidth]{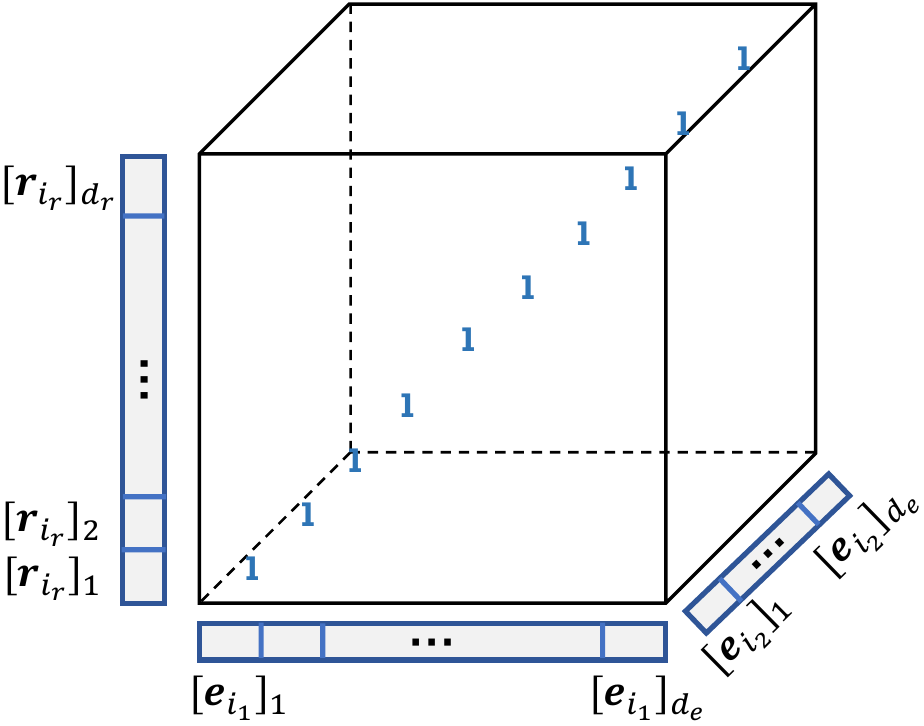}}
	\subfigure[ComplEx.]{\includegraphics[width=0.47\columnwidth]{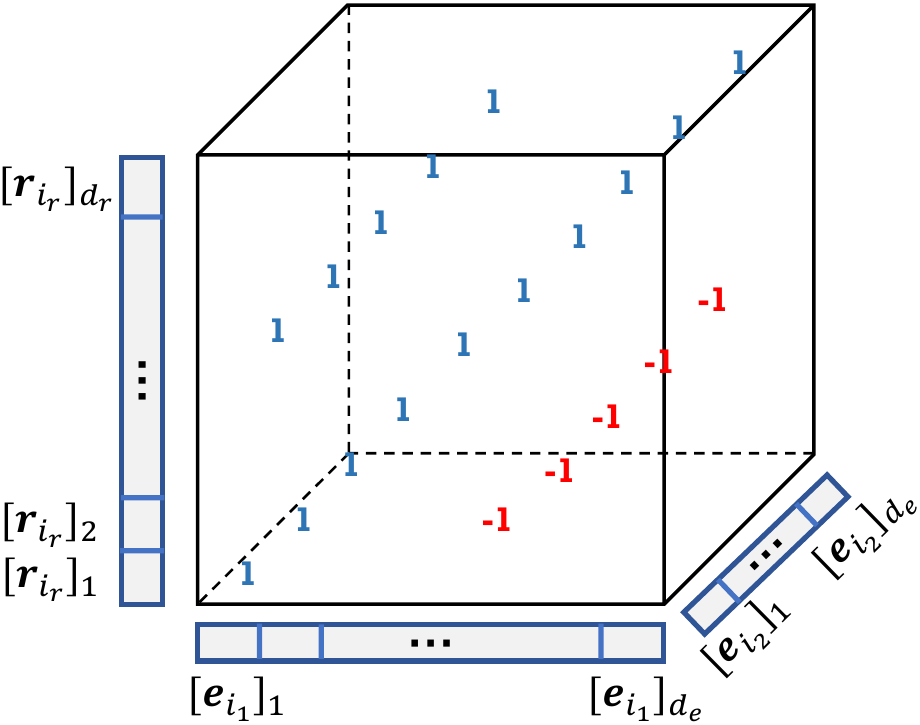}}
	\subfigure[SimplE.]{\includegraphics[width=0.47\columnwidth]{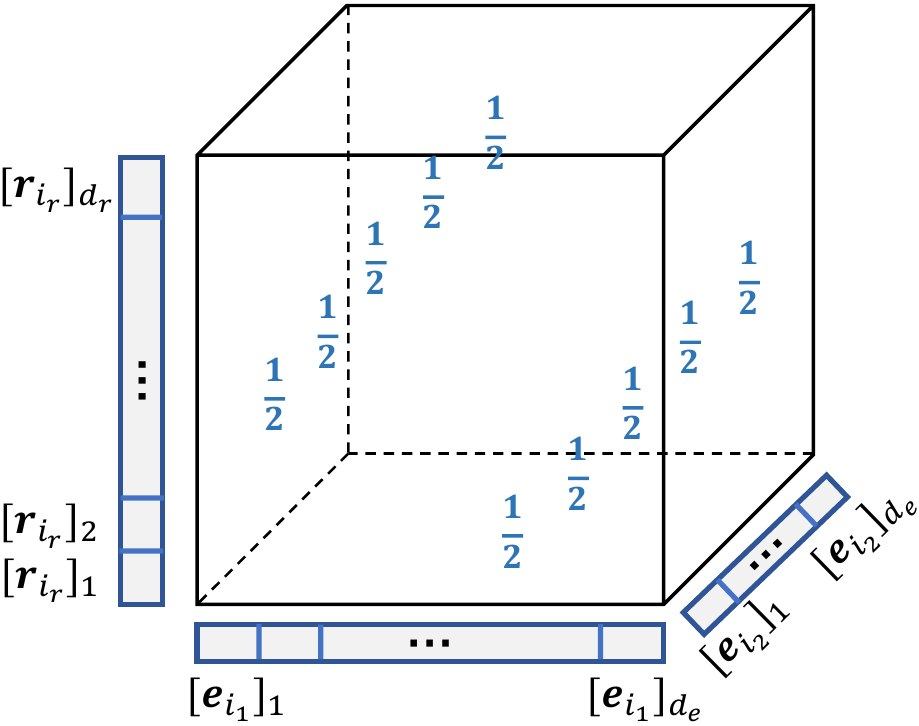}}
	\caption{
		(a) Each element in TuckER core tensor interprets the correlation between entities and relations of every embedding dimension; 
		(b), (c) and (d) illustrate DistMult, ComplEx and SimplE under representations of TuckER core tensor, respectively.
		Note that elements that are set to 0 are represented in white while gray elements are unknown.}
	\label{fig:coreTensor}
\end{figure*}

\section{Related Works}
Recently, many tensor decomposition approaches have been introduced to describe the N-ary relational data \cite{yang2014embedding,trouillon2017knowledge,kazemi2018simple,liu2017analogical,balazevic2019tucker,liu2020generalizing}.
Specifically,
given facts with a specific arity $n$, 
a $(n+1)$-order tensor $\mathcal{X}\in\{0,1\}^{n_r\times n_e\times\cdots\times n_e}$ is utilzed to represent a N-ary relational data, where $\mathcal{X}_{i_r,i_1,\dots,i_n}=1$ represents an existing fact $(r_{i_r},e_{i_1},\cdots,e_{i_n})$ otherwise $\mathcal{X}_{i_r,i_1,\dots,i_n}=0$.
For instance, binary relational data (i.e., $n=2$) is represented into 3-order tensor $\mathcal{X} \in \{0,1\}^{n_r\times n_e \times n_e}$.
Then, 
different tensor decomposition models differ in how the tensor $\mathcal{X}$ is decomposed into the entity embedding $\bm{E} \in \mathbb{R}^{n_e\times d}$, and relation embedding $\bm{R}\in \mathbb{R}^{n_r\times d}$.

Generally, 
there are two main tensor decomposition techniques that have been introduced to embed n-ary relational data,
i.e., CANDECOMP/PARAFAC (CP) decomposition \cite{hitchcock1927expression} and Tucker decomposition~\cite{tucker1966some}.
CP decomposes $\mathcal{X}$ as $\bm{R} \circ \bm{E} \circ\cdots\circ \bm{E}$,
and the scoring function measures the plausibility of
a n-ary fact $s=(r_{i_r},e_{i_1},\dots,e_{i_n})$ with embedding $\bm{H}=\{\bm{E},\bm{R}\}$ is
\begin{equation}
f(s, \bm{H}) 
= 
\left\langle 
\bm{r}_{i_r}, 
\bm{e}_{i_1},
\dots,
\bm{e}_{i_n}
\right\rangle.
\end{equation}
Tucker decomposition factorizes $\mathcal{X}$ as
$\mathcal G \times_1 \bm R\times_2 \bm E   \times_3 \cdots \times_{n+1} \bm E$, 
where $\mathcal G\in \mathbb{R}^{d_r\times d_e \times\cdots\times d_e}$.
Then, 
the corresponding scoring function  is
\begin{align}
\label{eq:tuckerSF}
f(s,\bm{H}) & = \mathcal G \times_1 \bm{r}_{i_r} \times_2 \bm{e}_{i_1} \times_3 \cdots \times_{n+1} \bm{e}_{i_n}.
\end{align}
Unlike CP, 
Tucker's core tensor $\mathcal{G}$ encodes the correlation between entity and relation embeddings.
Thus, 
the core tensor enables different entities and relations to share the same set of knowledge of any given N-ary relational data \cite{balazevic2019tucker}.

\subsection{Binary Relational Data Learning}
\label{ssec:relatedBinary}

In the past decades, embedding approaches have been developed as a promising method to handle binary relational data, such as translational distance models \cite{bordes2013translating,wang2014knowledge}, neural network models \cite{dettmers2018convolutional,balavzevic2019hypernetwork}, 
and tensor decomposition models \cite{yang2014embedding,trouillon2017knowledge,kazemi2018simple,liu2017analogical,balazevic2019tucker}.

As in Section~\ref{sec:intro}, the expressive capability is important for embedding models to achieve outstanding performance.
Among kinds of methods, tensor decomposition models demonstrate their superiority in terms of expressive guarantee \cite{kazemi2018simple,wang2018multi} and empirical performance \cite{lacroix2018canonical}.
More specifically, the literature \cite{yang2014embedding,trouillon2017knowledge,kazemi2018simple,liu2017analogical} have been shown to be different variants based on the CP decomposition \cite{lacroix2018canonical,zhang2019autosf}.
And TuckER \cite{balazevic2019tucker} first introduces Tucker decomposition \cite{tucker1966some,kolda2009tensor} into binary relational data learning.
Generally, the comprehensive core tensor design in TuckER can interpret CP-based tensor decomposition models (e.g., DistMult \cite{yang2014embedding}, ComplEx \cite{trouillon2017knowledge}, 
SimplE \cite{kazemi2018simple}) as sparse cases of various core tensors as illustrated in Figure~\ref{fig:coreTensor}. 
But please note that compared with TuckER, the CP-based tensor decomposition models \cite{yang2014embedding,trouillon2017knowledge,kazemi2018simple} show competitive performance in binary relational data without introducing the dense core tensor.
This motivates us to introduce the structured sparsity into high-order tensor decomposition models for N-ary relational data.

\begin{table}[h]
	\caption{A summary of common notations.}
	\label{table:notation}
	\setlength\tabcolsep{2pt}
	\small
	\centering
	\begin{tabular}{c|p{7cm}}
		\toprule
		\textbf{Symbol} & \textbf{Definition} \\ \midrule
		$s$  & The n-ary fact $s=(r_{i_r},e_{i_1},\dots,e_{i_n})$  \\ \midrule
		$\bm{E},\bm{R}$ & Embeddings $\bm{E} \in \mathbb{R}^{n_e\times d}, \bm{R}\in \mathbb{R}^{n_r\times d}$.
		\\ \midrule
		$f(s,\bm{H})$ & The scoring function of $s$ with $\bm{H}=\{\bm{E},\bm{R}\}$ \\ \midrule
		$M, N$ & The number of segments, and maximum arity in given data  \\ \midrule
		$\text{\tt OP}$ & Candidate diagonal tensor $\text{\tt OP} =\{-\mathcal{I}_{1}^n, \mathcal{I}_{0}^n, \mathcal{I}_{1}^n\}$ \\\midrule
		$\mathcal{Z}^n$ &The sparse core tensor for facts with arity $n$ \\ \midrule
		$\bm{\theta}$ & The core tensor weight\\ \midrule
		$\cdot$ & The vector dot product \\ \midrule
		$\langle \cdot \rangle$ & The multi-linear inner product, i.e., $\langle\bm{a},\bm{b},\bm{c} \rangle=\sum_{p=1}^d[\bm{a}]_p\cdot[\bm{b}]_p\cdot[\bm{c}]_p$\\ \midrule
		$\circ$ & The multi-way outer product, i.e., $(\bm{R}\circ\bm{E}\circ\bm{E})_{ijk} = \langle \bm{r}_i,\bm{e}_j,\bm{e}_k \rangle$\\ \midrule
		$\times_{k}$& The $k$-th mode product of $\mathcal{G}\in\mathbb{R}^{d_1\times\dots\times d_n}$ with $\bm{A}\in\mathbb{R}^{J\times d_k}$ , i.e., $(\mathcal{G}\times_{k}\bm{A})_{i_1,\dots,i_{k-1},j,i_{k+1},\dots,i_{n}} = \sum_{i_k=1}^{d_k}\mathcal{G}_{i_1,\dots,i_n}\bm{A}_{j,i_k}$. \\
		\bottomrule
	\end{tabular}
\end{table}

\subsection{N-ary Relational Data Learning}
\label{ssec:relatedNary}

As presented in Table~\ref{table:summary}, 
many models have been proposed to capture n-ary facts,
and tensor decomposition models are state-of-the-arts among them.
Specifically,
the core tensor of n-TuckER in \eqref{eq:tuckerSF} increases exponentially w.r.t the arity $n$.
To address such an over-parameterization problem,
GETD \cite{liu2020generalizing} simplifies $\mathcal{G}$ with the help of Tensor Ring Decomposition \cite{zhao2016tensor}, which can approximate the high-order tensor $\mathcal{G}$ by a set of 3-order latent tensors $\{\mathcal{W}_i\}$.
GETD first reshapes $\mathcal{G}$ into $c$-order tensor $\hat{\mathcal{G}}\in R^{d_1\times\cdots\times d_c}$ with $\prod_{i=1}^cd_i=d_e^nd_r$,
then decomposes $\hat{\mathcal{G}}$ into $c$ latent 3-order tensors 
$\{\mathcal{W}_i|\mathcal{W}_i\in R^{n_i\times d_i \times n_{i+1}}\}_{i=1}^c$, where $n_1 = \cdots = n_{c+1}$.
As a result, 
\eqref{eq:tuckerSF} is reformulated as
\begin{equation}\label{eq:GETD}
\mathcal{X} \approx \text{TR}\left(\mathcal{W}_1,\cdots,\mathcal{W}_c\right) \times_1 \bm R^{\top} \times_2 \bm E^{\top}   \times_3 \cdots \times_{n+1} \bm E^{\top},
\end{equation}
where $\text{TR}\left( \cdot \right)$ denotes the Tensor Ring computation \cite{zhao2016tensor,liu2020generalizing}.
The core tensor in GETD is subsequently reduced to $O( d_{\max}^3 )$, where $d_{\max}=\max_{i} d_i$.
However, it still requires cubic complexity, which is hard to train.
And note that $\mathcal{X}$ can only represent facts with a specific arity $n$.
Thus, existing tensor decomposition models suffer from the data sparsity issue since they cannot leverage all facts in n-ary relational data. 

\begin{figure*}[t]
\centering
\subfigure[Modeling sparse core tensor $\mathcal{Z}_2$.]{\includegraphics[width=0.7\columnwidth]{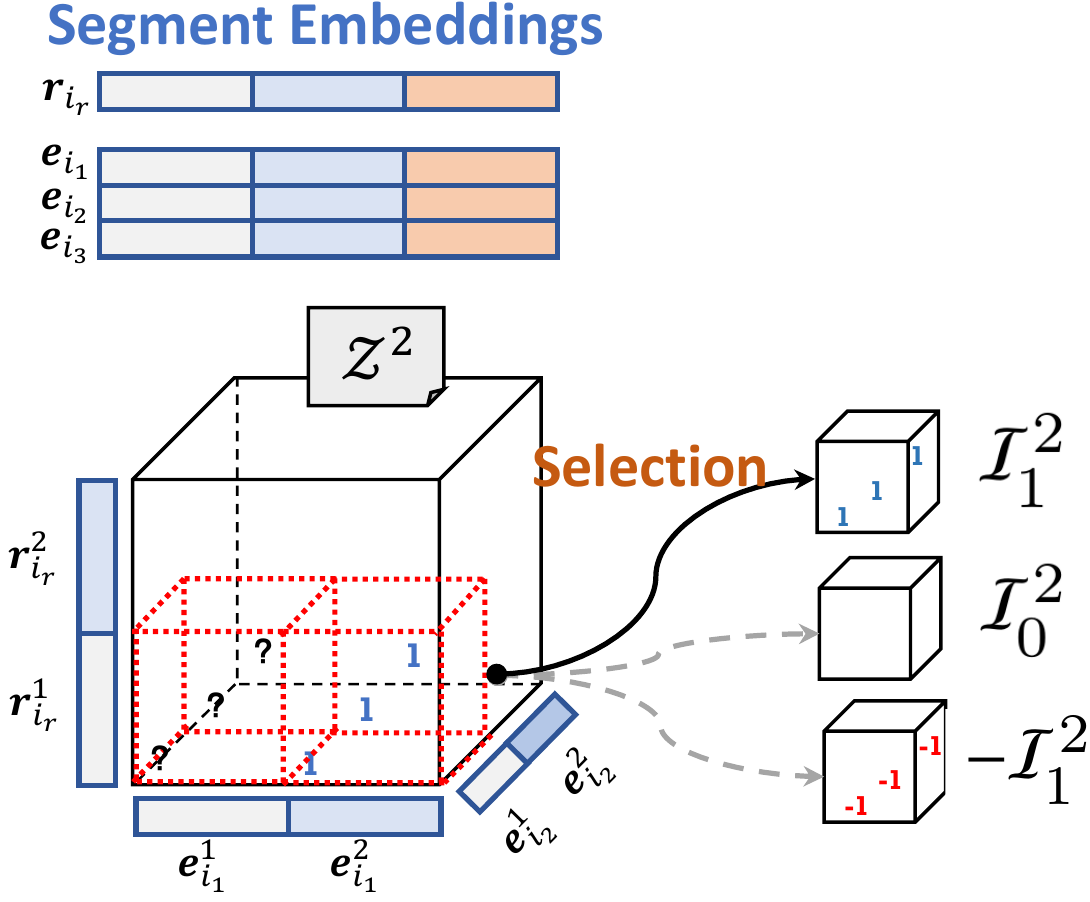}}
\hspace{5mm}
\subfigure[Modeling sparse core tensor $\mathcal{Z}_3$ .]{\includegraphics[width=1.2\columnwidth]{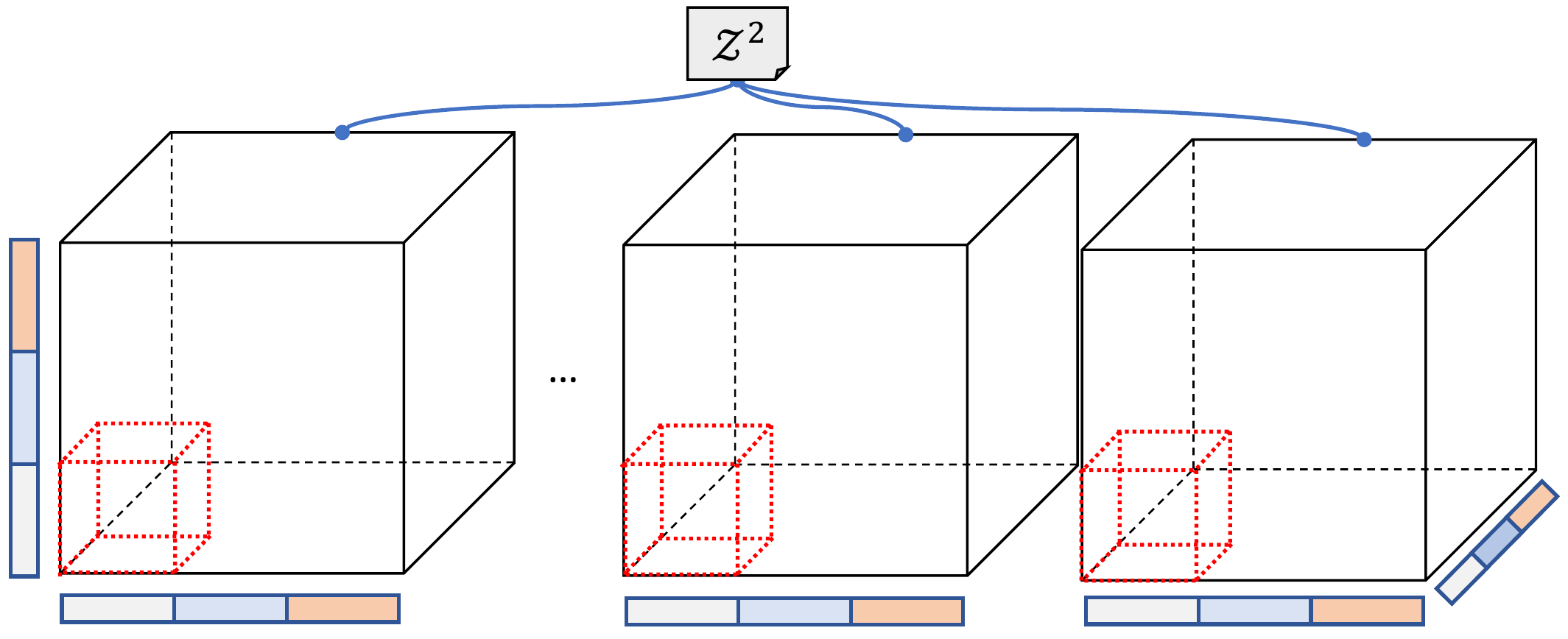}}
\caption{Illustration to sparsify core tensor. Set the number of segments $M=3$. (a) The embedding is segmented into $M$ parts. Then, for the binary fact, we only utilize first $2$-th embedding segments for computation and sparsify the core tensor with $\mathcal{Z}^2$, of which component is selected from $\{\mathcal{I}^2_1, \mathcal{I}^2_0, -\mathcal{I}^2_{1}\}$.
	(b) For $3$-ary fact, we employ all $3$ embedding segments for computation.
	Note that the calculation performed in the red cube is 
	$\mathcal{I}_0^3\times_1\bm{r}_{i_r}^{1}\times_2\bm{e}_{i_1}^{1}\times_3\bm{e}_{i_2}^{1}\times_4\bm{e}_{i_3}^{1}$.}
\label{fig:sparseCore}
\end{figure*}

\section{Reformulate Tensor Models}

Unfortunately, existing tensor decomposition models for the N-ary relational data still suffer from data-sparsity and over-parameterization (Section~\ref{sec:intro}).
First, $\mathcal{X}$ can only represent facts with a specific arity $n$ (Section~\ref{ssec:relatedNary}),
which limits existing models to only learn from facts with 
the fixed arity.
This makes the data-sparsity problem even more serious, 
as these models cannot fully leverage existing facts.
Besides, 
tensor decomposition models at least require a huge amount of parameters 
to be the expressive~\cite{liu2020generalizing}.
This makes them difficult to train and easy to overfit since there may not be enough training facts to activate the expressive power.
In the sequel, 
we propose a new tensor model based on sharing embedding (Section~\ref{ssec:embedShare})
and sparse core tensors (Section~\ref{ssec:sparseCore}) to address above issues.

\subsection{Share Embedding}
\label{ssec:embedShare}

As discussed in Section~\ref{ssec:relatedNary}, 
tensor decomposition models can only learn from 
the part of facts, 
i.e., facts with a specific arity $n$ in N-ary relational data, 
which causes more severe data sparsity issue.
Although they can be forced to jointly learn from facts with 
mixed arity by share the embedding across various arities \cite{wen2016representation,zhang2018scalable,guan2019link},
such embedding sharing 
scheme can be too restrictive and lead to poor performance.
Thus, to alleviate the data-sparsity issue, we propose to segment embeddings and share different embedding parts across arities for the		 N-ary relational data learning.

First, 
given the maximum arity $N$ and number of segments $M$ (usually $M\leq N\ll d$), we segment embeddings of relations and entities into 
$M$ splits, 
i.e., $\bm{e}_i = [\bm{e}^{1}_i; \dots; \bm{e}^{M}_i ]$ where $\bm{e}^{j}_i \in \mathbb{R}^{\nicefrac{d}{M}}$, and same for relation $\bm{r}_{i_r}$.
Then, given the arity $n$ and $m=\min\{n,M\}$, we utilize first $m$-th segments of embeddings to compute the score.
For example, given an entity vector $\bm{e}_i = [\bm{e}^{1}_i; \dots; \bm{e}^{3}_i]$, we use $[\bm{e}^{1}_i;\bm{e}^{2}_i]$ if it involves in a binary fact and use $[\bm{e}^{1}_i;\bm{e}^{2}_i;\bm{e}^{3}_i]$ for facts with arity 3 or even higher.
Then, 
to handle n-ary facts, we build a core tensor $\mathcal{Z}^n$ for every arity $n$, 
	where $\mathcal{Z}^n$ is a $(n+1)$-order tensor with size $\nicefrac{md}{M}$ (e.g., $\mathcal{Z}^2 \in \mathbb{R}^{\nicefrac{2d}{M}\times\nicefrac{2d}{M}\times\nicefrac{2d}{M}} $).
Overall, the proposed approach can handle the N-ary relational data with mixed arity by learning multiple core tensors $\{\mathcal{Z}^n\}_{n=2}^N$.
Such embedding sharing with segments can make embeddings learn from the low-order information in the high-order fact training, 
but also retain a part of the high-order specific information.

Unfortunately, 
each $\mathcal{Z}^n$ requires $O((\nicefrac{md}{M})^{n+1})$ and may still lead to over-parameterization.
Next, 
we introduce sparse core tensors that require much less complexity but maintains expressiveness.

\subsection{Sparsify Core Tensor}
\label{ssec:sparseCore}

Existing tensor decomposition models require a large number of parameters to maintain the expressiveness for the N-ary relational data, which makes the model inefficient and difficult to train.
Thus, the question comes that \textit{is it essential to learn a dense core tensor with so many trainable parameters for strong expressiveness}? 
To answer this question, we first review the domain-specific knowledge on binary relational data.

\subsubsection{Motivation from Binary Relational Data}\label{sec:moti}
TuckER introduces the dense core tensor $\mathcal{G}\in \mathbb{R}^{d_r\times d_e\times d_e}$ to achieve outstanding performance in binary relational data.
In \eqref{eq:tuckerSF},
each entry $\mathcal{G}_{k_r,k_1,k_2}$ in $\mathcal{G}$ actually interprets the correlation among embeddings at the dimension level, i.e., the $k_r$-th dimension of $\bm{r}$, $k_1$-th dimension of $\bm{e}_1$, and $k_2$-th dimension of $\bm{e}_2$.
However, 
such a redundant core tensor is hard to train and easy to overfit.

As mentioned in Section~\ref{ssec:relatedBinary}, 
other simple tensor-based models, such as 
ComplEx \cite{trouillon2017knowledge}, and SimplE \cite{kazemi2018simple}, can be regarded to have sparse core tensors with special patterns (see Figure~\ref{fig:coreTensor}).
But these simple models are expressive and achieve relatively good performance without introducing dense core tensor.
Consequently, it may be unnecessary to learn a smaller complex core tensor with an enormous amount of parameters in N-ary relational data.
This motivates us to sparsify the core tensor 
$\{\mathcal{Z}^n\}_{n=2}^N$ in the n-ary case by only interpreting the correlation among embedding segments.

\subsubsection{Structured Sparsity in Core Tensors}
\label{sssec:sparseCore}


We first divide the core tensor $\mathcal{Z}^n$ into 
$K = m^{n+1}$ tensors, 
denoted as $\mathcal{Z}^n = \{\mathcal{Z}_k^n\}_{k=1}^K$, where $\mathcal{Z}_k^n$ is a $(n+1)$-order tensor with size $\nicefrac{d}{M}$. 
After delving deep into tensor models on binary relational data (Figure~\ref{fig:coreTensor}), we observe that simple values (i.e., -1, 0 and 1) on the diagonal form of the core tensor are expressive for capturing interactions.
We first define such simple interaction in the high-order scenario.
A tensor $\mathcal{I}$ is \textit{diagonal} when $\mathcal{I}_{i,j,\dots,k}\neq 0$ holds if and only if $i=j=\dots=k$.
We use $\mathcal{I}_{v}^n$ to denote a $(n+1)$-order tensor with size $\nicefrac{d}{M}$, which is diagonal with $v$ on the super-diagonal and zeros elsewhere.
Then, we propose to select the appropriate diagonal tensor from $\{-\mathcal{I}_{1}^n, \mathcal{I}_{0}^n, \mathcal{I}_{1}^n \}$ to replace $\mathcal{Z}^n_k \in \mathcal{Z}^n$ as Figure~\ref{fig:sparseCore} (a).
Then, the diagonal tensor $\mathcal{I}_v^n$ encodes the correlation among embedding segments $(\bm{r}_{i_r}^{j_r}, \bm{e}_{i_1}^{j_1}, \dots, \bm{e}_{i_n}^{j_n})$, where $-\mathcal{I}_{1}^n$ represents the negative correlation, $\mathcal{I}_{0}^n$ is no correlation, and $\mathcal{I}_{1}^n$ denotes the positive correlation.
Note that any positive or negative value $v$ can be used for $\mathcal I_v^n$ here.
We utilize $1$ and $0$ for simplicity.
Formally, we formulate the definition of sparse core tensor as:

\begin{Definition}[Sparse Core Tensor]
\label{def:searchspace}
Given the embedding dimension $d$, 
the maximum arity $N$ and a specific arity $n$, let $\mathcal{I}^n_v$ denote the $(n+1)$-order diagonal tensor with size $\nicefrac{d}{M}$, and 
$\text{\tt OP}=\{-\mathcal{I}_1^n, \mathcal{I}_0^n, \mathcal{I}_1^n\}$ denote the operation set of candidate diagonal tensors.
Then, 
we propose to select every $\mathcal{Z}_k^n \in \mathcal{Z}^n$ from $\text{\tt OP}$.
Overall, the sparse core tensor is denoted to 
$\mathcal{Z}^n = \{\mathcal{Z}_k^n\}_{k=1}^K$, 
which interprets facts with the arity $n$.
\end{Definition}

Accordingly, given any fact $s$ with arity $n_s$, 
the scoring function based on $\mathcal{Z}^{n_s}$ is formulated as:
\begin{equation}
\label{eq:mix}
\!\!\!
f_z(s, \bm{H}; \mathcal{Z}^{n_s})
=\!\!\!\!\!\!\!\!\sum_{j_r,j_1,\dots,j_{n}}
\!\!\!\!\!\!\!\! \mathcal{Z}_k^{n_s} 
\!\times_1 \bm{r}_{i_r}^{j_r}
\!\times_2 \bm e_{i_1}^{j_1} 
\!\times_3 \cdots \times_{n_s+1} \bm{e}_{i_{n_s}}^{j_{n_s}},
\end{equation}
where any $j\in\{1,\dots,m\}$ and 
$k\in\{1,\dots,m^{n+1}\}$ corresponds to $(j_r, j_1,\dots,j_{n})$.
Compared with GETD's core tensor $O(cd^3_{\max})$, one sparse core tensor $\mathcal{Z}^n$ has a complexity of $O\left(m^{n+1}\right)$.
But note that $m, n \ll d_e$ or $d_r$,
and the arity $n$ over 4 are really rare in the common knowledge bases \cite{liu2020generalizing}.
Thus, we generally set the number of segments $M=4$ for the N-ary relational data in practical, which leads to a constant complexity such as $4^5=1,024$.
It is far smaller than the complexity of core tensor in GETD \cite{liu2020generalizing} in the real case (e.g., $4\cdot 50^3=500,000$).
And we theoretically demonstrate the expressiveness of S2S sparse core tensor design as in Theorem~\ref{theorem1}. The proof is presented in Appendix~\ref{appendix:theory}.

\begin{theorem}
\label{theorem1}
Given any N-ary relational data $S$ on the sets of entity $E$ and relation $R$, there exists a set of sparse core tensors $\{\mathcal{Z}^n\}_{n=2}^N$ with embeddings $\bm{E}$ and $\bm{R}$ that is able to accurately represent that ground truth.
\end{theorem}

In summary, we have enabled tensor decomposition models to learn from mixed arity
and maintained the expressiveness of core tensors with less model complexity.
However, it is still a non-trivial problem to design proper sparse core tensors $\{\mathcal{Z}^n\}_{n=2}^N$ due to a large number of candidates.
Recall that $\mathcal{Z}^n_k \in \mathcal{Z}^n$ can be arbitrarily and independently chosen from \texttt{OP} in Definition~\ref{def:searchspace}.
Assume that $M=4$, there are totally $3^{81}$ candidates for $\mathcal{Z}^3$.
In the next, we will introduce how to find proper sparse core tensors  
by leveraging the Neural Architecture Search (NAS) method.

\section{Search Algorithm}

In general,
the scoring function design should be a data-specific problem.
Since the N-ary relational data also owns specific prior-knowledge,
it is crucial to search for a set of proper sparse core tensors 
that can lead to outstanding performance on various N-ary relational data.

\subsection{Problem Formulation}

Continuous formulation~\cite{liu2018darts,yao2019differentiable} and stochastic formulation~\cite{xie2018snas,akimoto2019adaptive}
are two popular formulations in NAS literature,
they both model choices from a given operation set as 
a differentiable optimization problem.
The difference is that
continuous relaxation directly couples all candidate operations together,
while stochastic relaxation samples each candidate based on a learned distribution.

Considering that $-\mathcal{I}^n_{1}$ and $\mathcal{I}^n_{1}$ should not be coupled together since they 
are exactly the opposite, 
we follow stochastic relaxation and
sample $\mathcal{Z}_k^n$ independently and stochastically from \texttt{OP}.
For a $\mathcal{Z}^n = \{\mathcal{Z}^n_k\}_{k=1}^K$, let $\theta_{pk}^n$ denote the probability of $o_p\in \text{\tt OP}$ 
to be sampled for $\mathcal{Z}^n_k$, where $\sum_p \theta_{pk}^n = 1$.
Then, 
we utilize $\bm{\theta}^n = [\theta_{pk}^n]_{3\times K}$ maintain the probability weight for $\{\mathcal{Z}^n_k\}_{k=1}^K$, thus $\bm{\theta} = \{\bm{\theta}^n\}_{n=2}^N$ for all sparse core tensor $\{\mathcal{Z}^n\}_{n=2}^N$.
Moreover, we utilize $\mathcal{Z} = \{\mathcal{Z}^n\}_{n=2}^N$ to represent the sampled sparse core tensor from 
the categorical distribution $p_{\bm{\theta}}(\mathcal{Z})$.
Follow \cite{automl_book,yao2018taking,elsken2018neural}, 
we formulate the searching to sparsify core tensor problem as a bi-level optimization problem
in Definition~\ref{def:problem}.

\begin{Definition}[Search Problem]
\label{def:problem}
Given the training and validation facts $S_{\text{tra}}$ and $S_{\text{val}}$,
the sparse core tensor search problem is defined as follows:
\begin{align}
\label{eq:problem}
\!\!\!\!
\bar{\bm{\theta}}
=&
\arg \max\nolimits_{\bm{\theta}}
\mathbb{E}_{p_{\bm{\theta}}(\mathcal{Z})}
\big[
M
(
\bar{\bm{H}}, \mathcal{Z}; S_{\text{val}}
)
\big],
\\
\text{ s.t. }
\bar{\bm{H}}
\label{eq:problem2}
=&
\arg \min\nolimits_{\bm{H}}
\mathbb{E}_{p_{\bm{\theta}}(\mathcal{Z})}
\big[
L
\left(
\bm{H}, \mathcal{Z}; S_{\text{tra}}
\right)
\big].
\end{align}
\end{Definition}

Note that $L$ (resp. $M$) measures the loss 
(resp. mean reciprocal ranking~\cite{bordes2013translating,wang2014knowledge}) 
on the training (resp. validation) data.
The bi-level formulation in Definition~\ref{def:problem} is hard to optimize since 
both the embedding $\bm{H}$ and the sparse core tensor weight $\bm{\theta}$ 
are hierarchically coupled.
In the sequel,
we propose an efficient algorithm for optimization,
which is motivated by recent NAS algorithms~\cite{xie2018snas,akimoto2019adaptive}.

\subsection{Searching to Sparsify Core Tensor}

Finally, we summarize the algorithm of searching to sparsify core tensor in Algorithm~\ref{alg:S2S}, where embedding $\bm{H}$ and core tensor weight $\bm{\theta}$ are alternatively updated.
Alternating steepest ascent~\cite{liu2018darts,xie2018snas,akimoto2019adaptive,yao2019differentiable} 
is a way to avoid computationally heavy optimization \eqref{eq:problem} and \eqref{eq:problem2}.
For any sampled sparse core tensor $\mathcal{Z}$, 
we first optimize the embedding $\bm{H}$ on $\mathcal{Z}$ with a mini-batch data in steps~3-4.
Then, we evaluate the performance of sampled $\mathcal{Z}$ on the updated $\bm{H}$, which leads to a fast evaluation mechanism.
Thus, we are able to update the core tensor weight $\bm{\theta}$ every iteration in step~5-6.
After searching, we derive the most likely sparse core tensor $\{\bar{\mathcal{Z}}^n\}_{n=2}^N$ with the fine-tuned $\bar{\bm{\theta}}$ in step~8.
Finally, we learn the embedding $\bm{H}$ by training $\{\bar{\mathcal{Z}}^n\}_{n=2}^N$ from scratch in step~9.

\begin{algorithm}[ht]
	\caption{S2S: Searching to Sparsify Tensor Decomposition for N-ary relational data}
	\label{alg:S2S}
	\begin{algorithmic}[1]
		\STATE Initialize the embedding $\bm{H}$, 
		probability distribution $p_{\bm{\theta}}(\mathcal{Z})$.
		\WHILE{not converged}
		\STATE Randomly sample a mini-batch $B_{\text{tra}}$ from $S_{\text{tra}}$ and sparse core tensor set $\mathcal{Z}$ from $p_{\bm{\theta}}(\mathcal{Z})$;
		\STATE Update embeddings $\bm{H}$ with $\nabla_{\bm{H}} \mathbb{E}_{p_{\bm{\theta}}(\mathcal{Z})}\left[L\right]$
		in \eqref{eq:embedGradient};
		\STATE Randomly sample a mini-batch $B_{val}$ from $S_{\text{val}}$;
		\STATE Update the weight $\bm{\theta}$ with $\nabla_{\bm{\theta}} \mathbb{E}_{p_{\bm{\theta}}(\mathcal{Z})} 
		\left[
		M
		\right]$ in \eqref{eq:coreGradient};
		\ENDWHILE
		\STATE Derive final $\{\bar{\mathcal{Z}}^n\}_{n=2}^N$ from the fine tuned $\bar{\bm{\theta}}$, such as $\bar{\mathcal{Z}}^n_k = o_p$ where $p=\arg\max_{p} \theta^n_{pk}$;
		\STATE Achieve the final embedding $\bar{\bm{H}}$ by training embeddings with $\{\bar{\mathcal{Z}}^n\}_{n=2}^N$ from scratch to convergence. 
	\end{algorithmic}
\end{algorithm}

Given the distribution $p_{\bm{\theta}}(\mathcal{Z})$,
we propose to solve \eqref{eq:problem2} by minimizing the expected loss $L$ on the training data $S_{\text{tra}}$.
Then, 
stochastic gradient descent can be performed to optimize the embedding $\bm{H}$.
Based on Monte-Carlo (MC) sampling \cite{hastings1970monte}, 
we sample $\lambda$ core tensor sets to approximate the gradient $\nabla_{\bm{H}}$  as
\begin{align}
\label{eq:embedGradient}
\nabla_{\bm{H}} \mathbb{E}_{p_{\bm{\theta}}(\mathcal{Z})} 
\left[
L
\right]
\approx 
\frac{1}{\lambda}
\sum\nolimits_{i=1}^{\lambda}
\nabla_{\bm{H}}
L(\bm{H}, \mathcal{Z}^{(i)}; S_{\text{tra}})
,
\end{align}
where $\mathcal{Z}^{(i)}$ is a core tensor set that independent and identically distributed (i.i.d.) sampled from $p_{\bm{\theta}}(\mathcal{Z})$,
and $L(\bm{H}, \mathcal{Z}^{(i)}; S_{\text{tra}})$ is computed as:
\begin{equation}
L(\bm{H}, \mathcal{Z}^{(i)}; S_{\text{tra}}) 
= \sum\nolimits_{s\in S_{\text{tra}}} \ell 
\big(
s,f_z (\bm{H}; \mathcal{Z}^{n_s} ) 
\big),
\end{equation}
where $\ell(\cdot)$ is the extension of multi-class log-loss \cite{lacroix2018canonical} in the n-ary case \cite{liu2020generalizing} for a single fact $s$.
Similarly, the gradient w.r.t $\bm{\theta}$ can be approximated by MC sampling as:
\begin{equation}
\label{eq:coreGradient}
\!\!\!
\nabla_{\bm{\theta}} \mathbb{E}_{p_{\bm{\theta}}(\mathcal{Z})} 
\left[
M
\right]
\approx
\frac{1}{\lambda}
\sum\nolimits_{i=1}^{\lambda}
\nabla_{\bm{\theta}} M(\bm{H},\mathcal{Z}^{(i)};S_{\text{val}}).
\end{equation}
Then,
we propose to leverage
ASNG~\cite{akimoto2019adaptive},
which is the state-of-the-art stochastic optimization technique in NAS 
for optimizing $\bm{\theta}$:
\begin{equation*}
\nabla_{\bm{\theta}} M(\bm{H},\mathcal{Z}^{(i)};S_{\text{val}})
\!
=
\!
\sum_{s\in S_{\text{val}}}
\!\!
m
\big(s,f_z\left(\bm{H};\mathcal{Z}^{n_s}\right)\big)
\!
\left(
T\left(
\mathcal{Z}^{n_s}
\right)
\! - \! 
\bm{\theta}^{n_s}
\right),
\end{equation*}
where
$m(\cdot)$ measures the MRR performance on a single fact $s$ and $T(\cdot)$ denotes the sufficient statistic \cite{akimoto2019adaptive}.

\subsection{Comparison with AutoSF}
\label{sssec:autosf}

The closest work in the literature of the N-ary relational data is AutoSF~\cite{zhang2019autosf}, 
which proposes a NAS approach to search data-specific and bilinear scoring functions.
The proposed S2S differs from AutoSF from three perspectives:
task scenario, search space, and search algorithm.
AutoSF concerns the binary relational data based on the unified representation of embedding approaches.
We generalize the task scenario from the binary to N-ary relational data.
Correspondingly, we propose a novel search space where we can search for sparse core tensor in N-ary relational data.
And the search space of AutoSF is a special case of our proposed sparse core tensor.
Third, AutoSF develops an inefficient search algorithm, that requires training hundreds of candidates to convergence.
However, the N-ary relational data requires a much larger search space, which results in the efficiency issue become even more severe.
In this paper, we enable 
an efficient search algorithm 
ASNG \cite{akimoto2019adaptive} in our scenario, 
where the desired sparse core tensor can be searched by only training once.

\section{Experiments}
All codes are implemented with PyTorch and run on a single Nvidia RTX2080Ti GPU.

\subsection{Experimental Setup}
\subsubsection{Data Sets}
To demonstrate the performance of the proposed method, 
we conduct experiments on N-ary relational data with both various 
fixed arity (i.e., $n=2,3,4$) and mixed arity.
The statistics of data sets are summarized in Table~\ref{tab:data}.

\begin{itemize}[leftmargin=*]
\item
\textit{N-ary relational data.} 
We follow \cite{wen2016representation,zhang2018scalable,guan2019link,liu2020generalizing,rosso2020beyond} to compare various models on WikiPeople \cite{guan2019link} and JF17K \cite{wen2016representation}.
WikiPeople mainly concerns the entities of typing humans, which is extracted from Wikidata.
And JF17K is developed from Freebase \cite{bollacker2008freebase}.
Then, 
for $3$-ary and $4$-ary relational data,
we follow GETD \cite{liu2020generalizing} to filter out all $3$-ary and $4$-ary facts from WikiPeople and JF17K respectively, named as JF17K-3, JF17K-4, WikiPeople-3, and WikiPeople-4.

\item 
\textit{Binary relational data (aka. knowledge graph).} 
We follow~\cite{bordes2013translating,trouillon2017knowledge,kazemi2018simple,balazevic2019tucker,zhang2019autosf} to conduct experiments on four public benchmark data sets: 
WN18~\cite{bordes2013translating}, 
WN18RR~\cite{dettmers2018convolutional}, 
FB15k~\cite{bordes2013translating}, 
FB15k237~\cite{toutanova2015observed}.
WN18RR and FB15k237 are variants of WN18 and FB15k respectively by removing duplicate and inverse relations.
\end{itemize}

\begin{table}[h]
	\caption{Summary of benchmark N-ary relational data sets.}
	\label{tab:data}
	\setlength\tabcolsep{3pt}
	\centering
	\begin{tabular}{c|c|c|c|c|c|c}
		\toprule
		&  Data set   &   \#ent   & \#rel  & \#Tra & \#Val& \#Tst \\ \midrule
		&  WikiPeople-3   &   12,270   &  66 & 20,656  & 2,582 & 2,582\\
		fixed  &  WikiPeople-4  &   9,528   & 50  & 12,150  & 1,519 & 1,519 \\ 
		n-ary&  JF17K-3  &   11,541   & 104  & 27,635  & 3,454 & 3,455 \\ 
		&  JF17K-4  &    6,536  & 23  & 7,607  & 951 & 951 \\ \midrule
		mixed  &  WikiPeople   &   47,765   & 707  & 305,725  & 38,223 & 38,281\\
		n-ary&  JF17K   &   28,645   & 322  &  76,379 & - & 24,568 \\ \bottomrule
		&  WN18   &   40,943   &  18 & 141,442  & 5,000 & 5,000 \\ 
		binary &  WN18RR   &   40,943   &  11 & 86,835  & 3,034 & 3,134 \\ 
		&   FB15k  &  14,951   & 1,345 &  484,142 &50,000 &59,071 \\
		&   FB15k237  &  14,541  & 237  &  272,115 & 17,535 & 20,466\\ \midrule
	\end{tabular}
\end{table}

\begin{table*}[t]
	\centering
	\caption{The link prediction results on the WikiPeople-3/4. 
	}
	\label{table:linkPredictionNary1}
	\begin{tabular}{c|c|cccc|cccc}
		\toprule
		\multirow{2}{*}{model type}&
		\multirow{2}{*}{model}         &      \multicolumn{4}{c|}{WikiPeople-3}       &      \multicolumn{4}{c}{WikiPeople-4}         \\
		&&        MRR     &      Hits@1   &  Hits@3   &  Hits@10      &        MRR    &      Hits@1   &Hits@3 &      Hits@10         \\ \midrule
		translation &RAE \cite{zhang2018scalable}   &  0.239 & 0.168 & 0.252 & 0.379 &0.150 & 0.080 & 0.149 & 0.273  \\ \midrule
		&NaLP~\cite{liu2020generalizing}    &  0.301 & 0.226 & 0.327 & 0.445 & 0.342 & 0.237 & 0.400 & 0.540  \\ 
		neural network & HINGE \cite{rosso2020beyond} & 0.338 & 0.255 & 0.360 & 0.508 & 0.352 & 0.241 & 0.419 & 0.557 \\
		& NeuInfer \cite{guan2020neuinfer} & 0.355 & 0.262 & 0.388 & 0.521 & 0.361 & 0.255 & 0.424 &  0.566 \\ \midrule
		&n-CP~\cite{liu2020generalizing}    &  0.330 & 0.250 & 0.356 & 0.496 & 0.265 & 0.169& 0.315 & 0.445  \\ 
		tensor&n-TuckER~\cite{liu2020generalizing}    & 0.365  & 0.274& 0.400 & 0.548 &0.362 & 0.246 & 0.432 & 0.570  \\ 
		decomposition&GETD~\cite{liu2020generalizing}   &  \underline{0.373} & \underline{0.284} & \underline{0.401}  & \underline{0.558} & \underline{0.386} &\underline{0.265}  &\underline{0.462}  &\underline{0.596}   \\ \cmidrule{2-10}
		&S2S      &  \textbf{0.386 } & \textbf{0.299}  &\textbf{0.421}   & \textbf{0.559 }   &\textbf{0.391}   & \textbf{0.270}  & \textbf{0.470}   &\textbf{0.600 }   \\ 
		\bottomrule
	\end{tabular}
\end{table*}

\begin{table*}[t]
	\centering
	\caption{The link prediction results on the JF17K-3/4.}
	\label{table:linkPredictionNary2}
	\begin{tabular}{c|c|cccc|cccc}
		\toprule
		\multirow{2}{*}{model type}         &
		\multirow{2}{*}{model}         &      \multicolumn{4}{c}{JF17K-3}      &     \multicolumn{4}{c}{JF17K-4}      \\
		&&        MRR    &      Hits@1    & Hits@3 &    Hits@10      &        MRR    &      Hits@1    &  Hits@3   & Hits@10     \\ \midrule
		translation &RAE \cite{zhang2018scalable}  & 0.505 & 0.430  & 0.532 & 0.644 & 0.707  & 0.636 & 0.751 & 0.835  \\ \midrule
		&NaLP~\cite{liu2020generalizing}  & 0.515 & 0.431 & 0.552 & 0.679 & 0.719  & 0.673& 0.742 &  0.805 \\ 
		neural network & HINGE \cite{rosso2020beyond} & 0.587 & 0.509 & 0.621 & 0.738 & 0.745 & 0.700 & 0.775 & 0.842  \\
		& NeuInfer \cite{guan2020neuinfer} & 0.622 & 0.533 & 0.658 &  0.770 & 0.765 & 0.722 & 0.808 & 0.871 \\ \midrule
		&n-CP~\cite{liu2020generalizing}  & 0.700 & 0.635 & 0.736 & 0.827 &  0.787 & 0.733 & 0.821 &0.890  \\ 
		tensor  &n-TuckER~\cite{liu2020generalizing}    & 0.727 & 0.664& 0.761 & 0.852 & 0.804  & 0.748 & 0.841 &0.902  \\ 
		decomposition &GETD~\cite{liu2020generalizing}  & \underline{0.732}  & \underline{0.669 } & \underline{0.764}  & \underline{0.856}  & \underline{0.810}  & \underline{0.755}  & \underline{0.844}  & \underline{0.913}   \\\cmidrule{2-10}
		& S2S    & \textbf{0.740}   & \textbf{0.676} & \textbf{0.770}  &\textbf{0.860}  & \textbf{0.822}  & \textbf{0.761} & \textbf{0.853}  & \textbf{0.924}  \\ 
		\bottomrule
	\end{tabular}
\end{table*}

\subsubsection{Evaluation Metrics}

We test the performance of our proposed method on the link prediction task \cite{zhang2019autosf,zhang2019quaternion}, which is utilized to complete the N-ary relational data.
Given a n-ary fact $s=(r_{i_r},e_{i_1},\dots,e_{i_n})$, the embedding model assumes one entity in this fact is missing, then it ranks all candidate entities based their scores.
We adopt the standard metrics~\cite{bordes2013translating,wang2014knowledge}:
\begin{itemize}[leftmargin=*]
\item Mean Reciprocal Ranking (MRR):
$\nicefrac{1}{|S|}\sum_{i=1}^{|S|}\nicefrac{1}{\text{rank}_i}$, where $\text{rank}_i$ is the ranking result, and 

\item Hits@$T$:  $\nicefrac{1}{|S|}\sum_{i=1}^{|S|}\mathbb{I}(\text{rank}_i \le T)$, where $\mathbb{I}(\cdot)$ is the indicator function and $T\in \{1,3,10\}$.
\end{itemize}
Note that the higher MRR and Hits@$T$ indicate the better quality of embeddings.
And all metrics are reported in a ``filter'' setting~\cite{bordes2013translating}, where the ranking computation is not include the corrupted facts that exist in train, valid and test data sets.

\subsubsection{Hyper-parameter Settings}
The proposed method mainly contains two steps, searching for sparse core tensor, and training the searched core tensor to convergence.
In the search strategy, we utilize the default hyper-parameters implemented in ASNG~\cite{akimoto2019adaptive} for optimizing the core tensor weight.
Then, 
we train the embeddings on the searched hyper-parameter set, which is achieved by tuning CP/n-CP \cite{lacroix2018canonical} with the help of HyperOpt \cite{bergstra2013making}.
This hyper-parameter set includes the learning rate, decay rate, batch size, and embedding dimension.
Besides, we optimize the embedding with Adam algorithm \cite{kingma2014adam}.
To determine the sparse core tensor for evaluation, 
we run S2S five times and report average results.

\subsection{N-ary Relational Data with Fixed Arity}

We first compare our S2S with other models in N-ary relational data with fixed arity,
i.e., WikiPeople-3, WikiPeople-4, JF17K-3, and JF17K-4.
We adopt the n-ary tensor decomposition models, n-CP \cite{lacroix2018canonical}, n-TuckER \cite{balazevic2019tucker}, and GTED \cite{liu2020generalizing}, as baselines.
As for the translational model, we only include the advanced RAE \cite{zhang2018scalable} since it is an upgraded version of m-TransH \cite{wen2016representation}.
And we also compare the neural network models NaLP \cite{guan2019link}, HINGE \cite{rosso2020beyond}, and NeuInfer~\cite{guan2020neuinfer}.

\subsubsection{Benchmark Comparison}
We demonstrate the performance on N-ary relational data with fixed arity in 
Table~\ref{table:linkPredictionNary1}-\ref{table:linkPredictionNary2}.
We can observe that tensor decomposition models (n-CP, n-TuckER, GETD, and S2S) generally have better performance than other models
in Table~\ref{table:linkPredictionNary1}-\ref{table:linkPredictionNary2}.
That is mainly because tensor decomposition models have strong expressiveness.
Then, although n-CP requires the lowest complexity $O(n_ed_e+n_rd_r)$ among tensor decomposition models, it does not achieve the high performance as other tensor decomposition models (e.g., n-TuckER, GETD, and S2S).
That is because n-CP does not introduce a core tensor like tensor decomposition models, which can enable the embedding to share the domain knowledge.
Furthermore, we can observe that GETD performs better than n-TuckER since GETD partially addresses the over-parameterized problem in n-TuckER.
Overall, our proposed S2S consistently achieves state-of-the-art performance on all benchmark data sets by the data-specific core tensor design.

\begin{table*}[t]
	\centering
	\caption{The link prediction results on the multi-relational data set with mixed arity.}
	\label{table:linkPredictionMix}
	\begin{tabular}{c|cccc|cccc}
		\toprule
		\multirow{2}{*}{model}         &      \multicolumn{4}{c|}{WikiPeople}       &      \multicolumn{4}{c}{JF17K}   \\
		&        MRR     &      Hits@1   &  Hits@3 &   Hits@10      &        MRR    &      Hits@1    &  Hits@3   & Hits@10      \\ \midrule
		RAE \cite{zhang2018scalable}    & 0.172 & 0.102 & 0.182 & 0.320 & 0.310 & 0.219 & 0.334 & 0.504 \\ 
		NaLP \cite{guan2019link} & 0.338 & 0.272 & 0.364 & 0.466 & 0.366 & 0.290 & 0.391 & 0.516 \\ 
		HINGE \cite{rosso2020beyond} & 0.333  & 0.259 & 0.361 & 0.477 & 0.473 & 0.397 & 0.490 & 0.618 \\ 
		NeuInfer \cite{guan2020neuinfer} & \underline{0.350} & \textbf{0.282} & \underline{0.381} & 0.467 & \underline{0.517}  & \underline{0.436}  & \underline{0.553}  & \underline{0.675}  \\ 
		HypE \cite{fatemi2019knowledge} & 0.292 & 0.162 & 0.375 & \underline{0.502}  & 0.494 & 0.408 & 0.538 & 0.656 \\ \midrule
		S2S        &  \textbf{0.372}   &\underline{0.277} & \textbf{0.439}   & \textbf{0.533}   & \textbf{0.528}   & \textbf{0.457}   & \textbf{0.570}   & \textbf{0.690}   \\ \bottomrule
	\end{tabular}
\end{table*}

\begin{figure*}[ht]
	\centering
	\subfigure[WikiPeople-3.]
	{\includegraphics[width=0.25\linewidth]{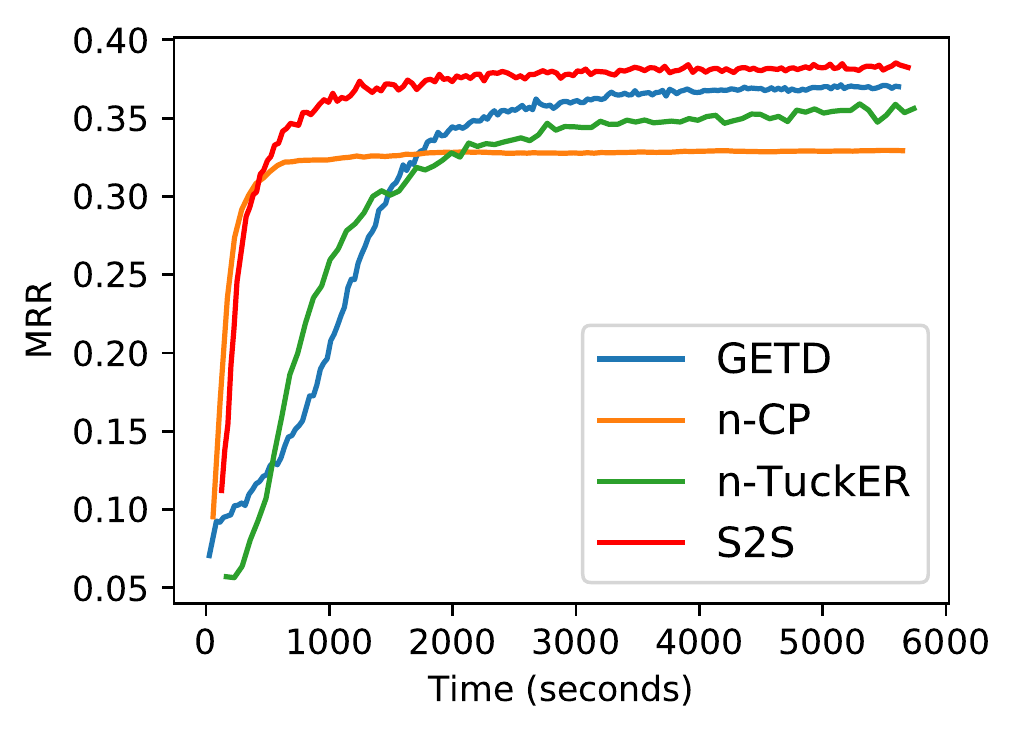}}
	\subfigure[JF17K-3.]
	{\includegraphics[width=0.235\linewidth]{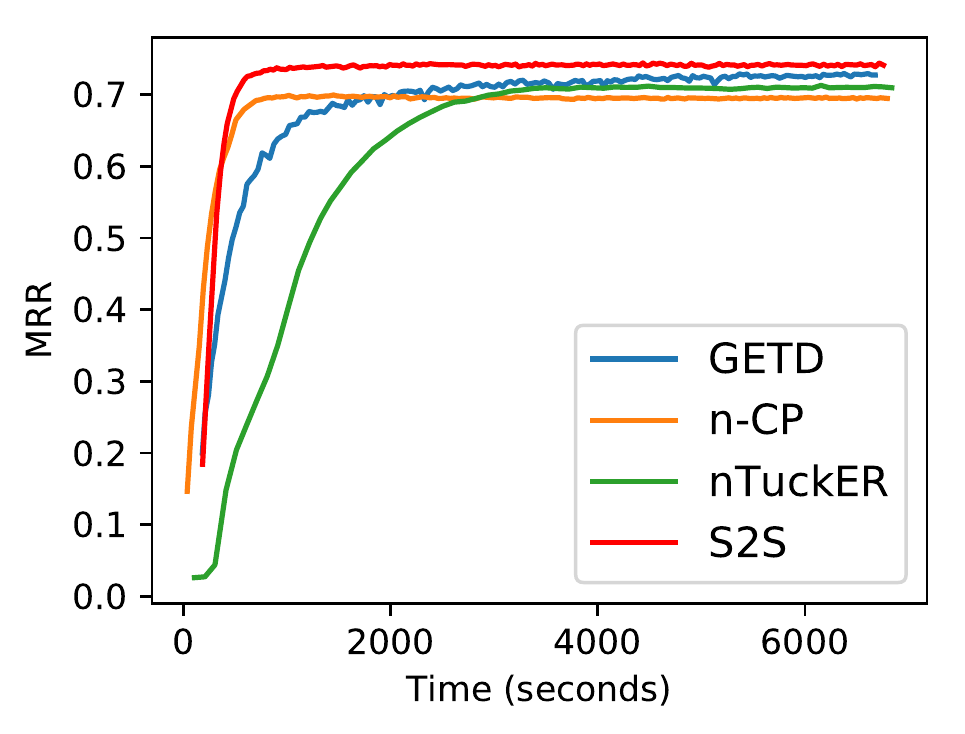}}
	\subfigure[WikiPeople-4.]
	{\includegraphics[width=0.24\linewidth]{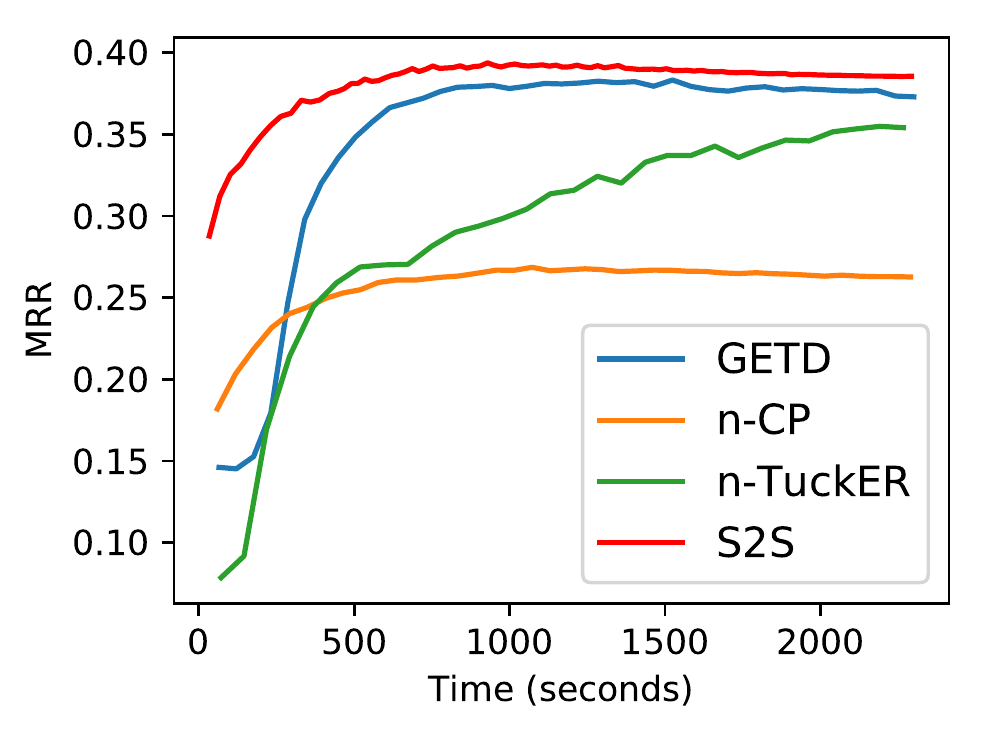}}
	\subfigure[JF17K-4.]
	{\includegraphics[width=0.238\linewidth]{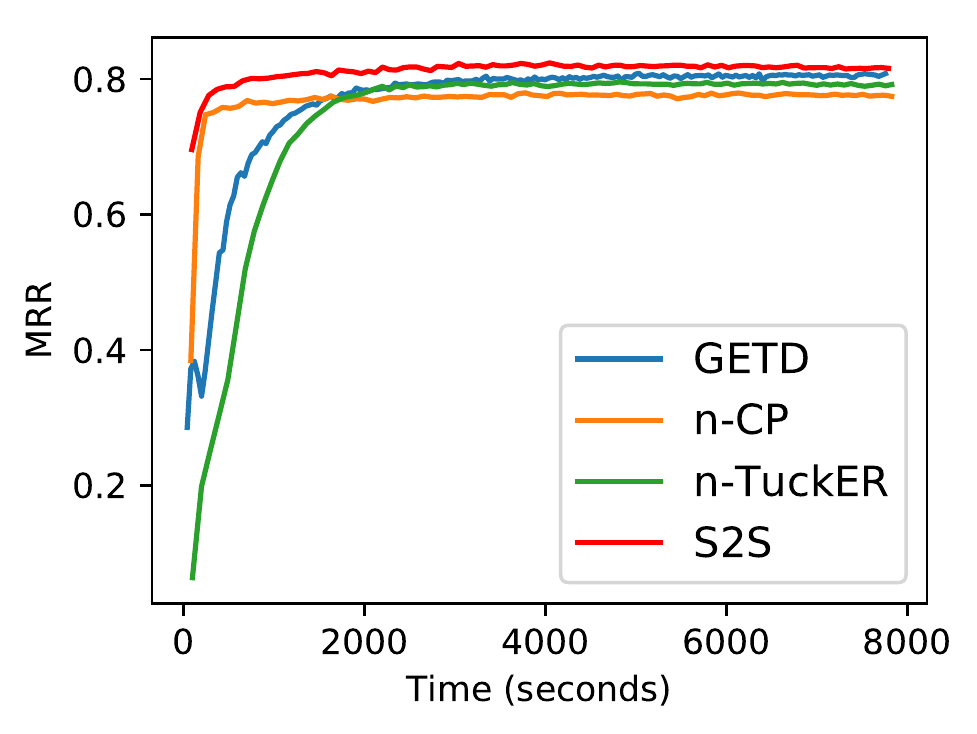}}
	\caption{Testing MRR v.s. clock time (seconds) with fixed arity.}
	\label{figs:exp:LC-fixed}
\end{figure*}

\begin{figure}[!t]
	\centering
	\subfigure[WikiPeople.]
	{\includegraphics[width=0.495\linewidth]{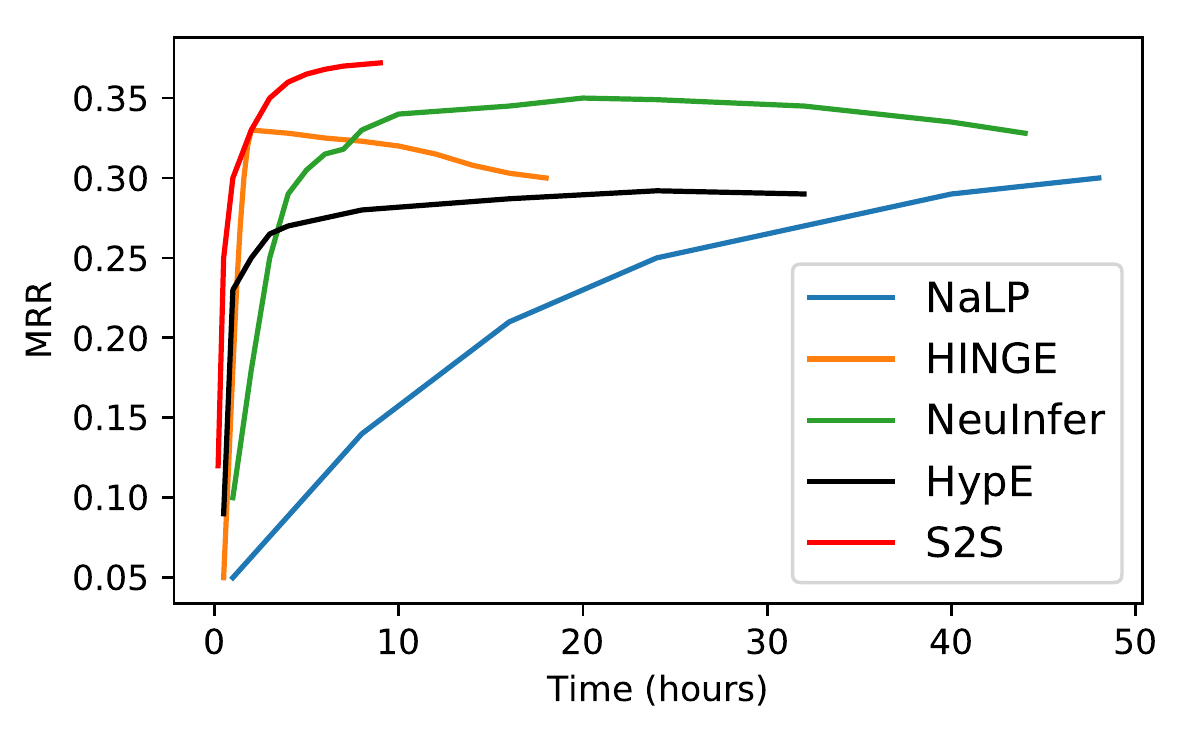}}
	\subfigure[JF17K.]
	{\includegraphics[width=0.49\linewidth]{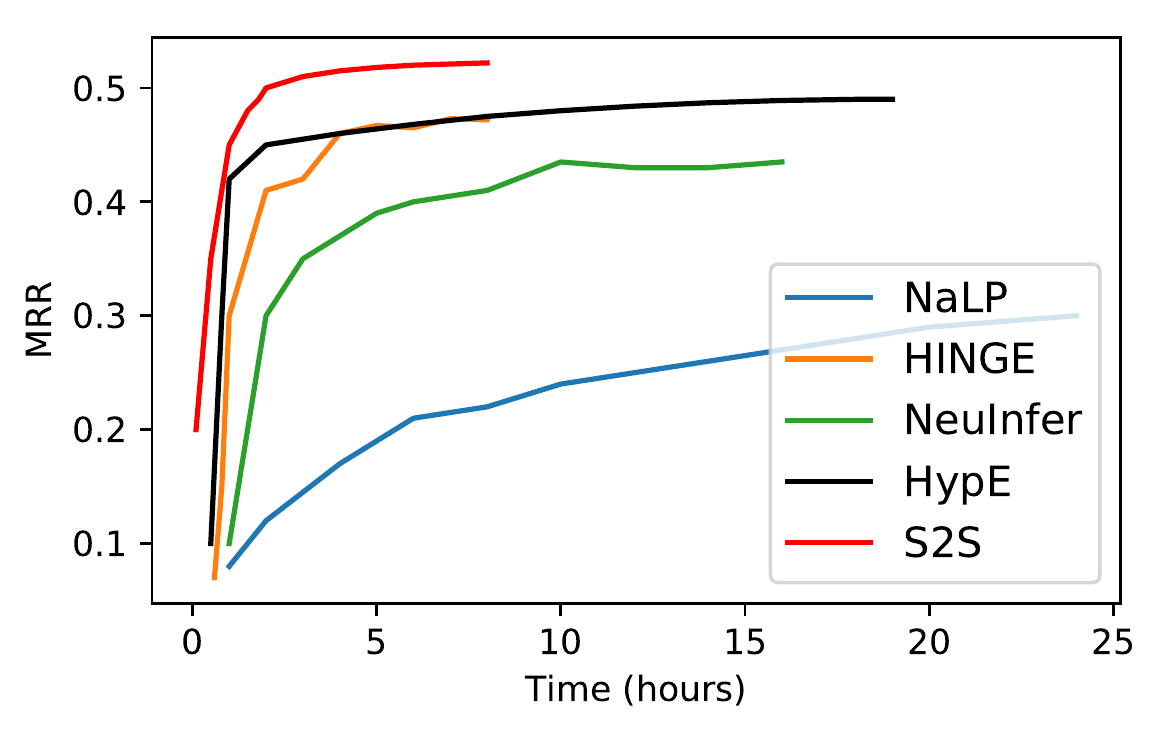}}
	\caption{Testing MRR v.s. clock time (hours) with mixed arity.}
	\label{figs:exp:LC-Mix}
\end{figure}

\subsubsection{Training Efficiency}
Moreover, we show the learning curve of several tensor decomposition models to compare the efficiency in 
Figure~\ref{figs:exp:LC-fixed}.
n-CP converges fastest due to the lowest model complexity.
The convergence rate of n-TuckER is the slowest since it requires the most complexity.
GETD converges much faster than n-TuckER because it reduces the complexity of the core tensor.
And the convergence of S2S is only slower than that of n-CP and faster than GETD and n-TuckER due to our sparse core tensor design.

\subsection{N-ary Relational Data with Mixed Arity}

To demonstrate the importance of mixed arity and superiority of our S2S, 
we compare it with other advanced models on the N-ary relational data,
i.e., Wiki-People \cite{guan2019link} and JF17K \cite{wen2016representation}.
We include the advanced translational model RAE \cite{zhang2018scalable}, the neural networks models NaLP \cite{guan2019link}, HINGE \cite{rosso2020beyond} and NeuInfer \cite{guan2020neuinfer}, and a hybrid model HypE \cite{fatemi2019knowledge}.

\subsubsection{Benchmark Comparison}
We show the performance on N-ary relational data with mixed arity in Table~\ref{table:linkPredictionMix}.
Because of lack of expressive ability, the translational model RAE does not achieve good performance.
The neural network models \cite{guan2019link,rosso2020beyond,guan2020neuinfer} generally outperform the translational model RAE by leveraging complex networks.
On the contrary, S2S leads to state-of-the-art performance because of the expressive guarantee.

\subsubsection{Training Efficiency}
In 
Figure~\ref{figs:exp:LC-Mix}, it is obvious that the neural network models, 
i.e., NaLP \cite{guan2019link} and NeuInfer \cite{guan2020neuinfer}, 
require quite a long time to convergence.
That is because these two models 
utilize complex neural networks for training.
On the contrary, another neural network model HINGE \cite{rosso2020beyond} proposes a simple way to train the embeddings, which converges much fast.
Among all models, S2S achieves the fastest convergence since it requires less complexity with the sparse core tensor.

\subsection{Binary Relational Data}

To further demonstrate the performance of the proposed method, we also compare S2S with classical embedding approaches on binary relational data,
i.e.,
WN18~\cite{bordes2013translating}, 
WN18RR~\cite{dettmers2018convolutional}, 
FB15k~\cite{bordes2013translating}, 
FB15k237~\cite{toutanova2015observed}.
We include the most advanced translational model RotatE \cite{sun2019rotate} due to its outstanding performance among translational models.
We also compare two popular neural network models, ConvE \cite{dettmers2018convolutional} and HypER \cite{balavzevic2019hypernetwork}.
As for tensor-based models, we include DistMult \cite{yang2014embedding}, ComplEx \cite{trouillon2017knowledge}, SimplE \cite{kazemi2018simple}, HolEX \cite{xue2018expanding},  QuatE~\cite{zhang2019quaternion}, and TuckER \cite{balazevic2019tucker}.
Moreover, we include the recent scoring function search method, AutoSF \cite{zhang2019autosf}, which only concerns the binary relational data as mentioned in Section~\ref{sssec:autosf}.

\subsubsection{Benchmark Comparison}
The ranking performance is in Table~\ref{table:linkPrediction}. 
It is clear that classical models cannot consistently achieve good performance on various data sets,
since these models are not data-specific.
AutoSF can search for a suitable scoring function for each data set and consistently achieve outstanding performance.
The proposed S2S is also data-specific, which aims to search proper sparse core tensor for any given data.
Overall, S2S consistently achieves state-of-the-art performance in all data sets.

\begin{table*}[t]
	\centering
	\caption{Comparison of the proposed S2S and state-of-the-art scoring functions on the link prediction task.}
	\label{table:linkPrediction}
	\setlength\tabcolsep{2.5pt}
	\begin{tabular}{c|c|ccc|ccc|ccc|ccc}
		\toprule
		\multirow{2}{*}{model} & \multirow{2}{*}{model}         &      \multicolumn{3}{c|}{WN18}       &      \multicolumn{3}{c|}{WN18RR}      &      \multicolumn{3}{c|}{FB15k}      &     \multicolumn{3}{c}{FB15k237}      \\ 
		& &        MRR     &      Hits@1   &      Hits@10      &        MRR    &      Hits@1    &      Hits@10       &        MRR    &      Hits@1    &      Hits@10      &        MRR    &      Hits@1     &      Hits@10      \\ \midrule
		translation &RotatE~\cite{sun2019rotate}       &       0.949   &  0.944   &  0.959   &       0.476   & 0.428    &      \underline{0.571}   &   0.797   &     0.746       &       0.884       &       0.338     &  0.241  &       0.533   \\ \midrule
		neural &ConvE~\cite{dettmers2018convolutional} &       0.943    & 0.935  &       0.956       &       0.460    &  0.390  &       0.480        &       0.754   &  0.670   &       0.873       &       0.316   &  0.239    &       0.491       \\ 
		network&HypER \cite{balavzevic2019hypernetwork} & 0.951 & \underline{0.947} & 0.958	  & 0.465    &0.436  &  0.522   & 0.790  & 0.734 &  0.885   &   0.341  & 0.252 &  0.520  \\ \midrule
		&HolEX~\cite{xue2018expanding}      &       0.938     & 0.930  &       0.949       &         -     &   -  &         -         &0.800 &      0.750        &       0.886       &         -     &    -  &        -         \\
		&QuatE~\cite{zhang2019quaternion}    &       0.950     & 0.945  &  0.959 &0.488 & 0.438 &   \textbf{0.582}   & 0.782 &  0.711 &       0.900       &       0.348      & 0.248  &       0.550       \\
		tensor&DistMult~\cite{yang2014embedding}    &       0.821      & 0.717  &       0.952       &       0.443    & 0.404   &       0.507        &       0.817  &     0.777 &       0.895       &       0.349   &   0.257   &       0.537       \\
		decomposition&ComplEx~\cite{trouillon2017knowledge}  & 0.951 &  0.945    &  0.957       &       0.471       & 0.430   &    0.551        &       0.831    &  0.796  & \underline{0.905} &       0.347     &  0.254  &       0.541       \\
		&SimplE~\cite{kazemi2018simple}     &       0.950   &   0.945  &  0.959  &       0.468      & 0.429  &       0.552        &       0.830   &  0.798  &       0.903       &       0.350     &  0.260  &       0.544       \\
		&TuckER~\cite{balazevic2019tucker}    &  \underline{0.953}   & \textbf{0.949 }&  0.958 &       0.470       &  0.443 &    0.526        &       0.795       & 0.741 &      0.892       & 0.358 & 0.266 &  0.544 \\ 
		&GETD~\cite{liu2020generalizing}             &   0.948   & 0.944 & 0.954 & -& - & - &  0.824& 0.787& 0.888 &  - & - & -  \\ \midrule
		NAS &AutoSF \cite{zhang2019autosf} &  0.952  & \underline{0.947} & \underline{0.961}  &  \underline{0.490}   &\underline{0.451}  &  0.567   & \textbf{0.853} & \textbf{0.821} &   \textbf{0.910} &   \underline{0.360}  & \underline{0.267} & \underline{0.552}    \\ \cmidrule{2-14}
		& S2S        &       \textbf{0.955}  &  \textbf{0.949}    &  \textbf{0.963}   &  \textbf{0.498} & \textbf{0.455}  & \underline{0.577} &  \underline{0.850} & \underline{0.820}  &  \textbf{0.910}   &   \textbf{0.368}   & \textbf{0.270} &  \textbf{0.559 }  \\ \bottomrule
	\end{tabular}
\end{table*}

\subsection{Search Efficiency}

To investigate the search efficiency of the 	proposed method, 
we summarize the running time of S2S and other models on 4 binary data sets in 
Table~\ref{table:time}.
We compare S2S with AutoSF in terms of the score function search time, and stand-alone training time of searched score function.
Note that S2S sets the embedding dimension to 512 in the search procedure for all data sets.
As for stand-alone training, we set embedding dimension for all models at 1024.
We utilize the simplest tensor decomposition model DistMult~\cite{yang2014embedding} as the benchmark.
In stand-alone training, S2S significantly reduces the training time compared with TuckER since it sparsifies the core tensor of TuckER.
And the training time of the scoring function searched by S2S is a little longer than DistMult.
That is because S2S searches a slightly more complex core tensor than DistMult's as illustrated in Figure~\ref{fig:coreTensor} (b) and Figure~\ref{fig:sparseCore} (a).
Compared with another search approach AutoSF, S2S significantly reduces the search cost.
AutoSF adopts the stand-alone evaluation mechanism, which requires training the hundreds of candidate scoring functions to convergence.
But the proposed S2S enables an efficient search algorithm ASNG \cite{akimoto2019adaptive}, where the proper scoring function can be searched by only training once (i.e., one-shot manner).
Furthermore, S2S searches only take a bit more time than DistMult since it needs to update the architecture parameter in search.
In summary, the proposed method is very efficient in terms of search and stand-alone training.

\begin{table}[!t]
\caption{Running time (in hours) analysis of several models.}
\label{table:time}
\small
\setlength\tabcolsep{2pt}
\centering
\begin{tabular}{c|c|c|c|c|c|c}
	\toprule
	\multirow{2}{*}{data set} & \multirow{2}{*}{DistMult} &  \multicolumn{2}{c|}{S2S}  & \multicolumn{2}{c|}{AutoSF}  & \multirow{2}{*}{TuckER} \\
	     \cmidrule{3-6}       &                           &   Search    &   Training   &    Search     &   Training   &                         \\ \midrule
	          WN18            &        1.9$\pm$0.1        & 2.0$\pm$0.2 & 2.4$\pm$0.1  & 65.7$\pm$3.0  & 2.4$\pm$0.1  &      25.4$\pm$1.5       \\ \midrule
	         WN18RR           &        0.4$\pm$0.1        & 1.3$\pm$0.1 & 0.6$\pm$0.1  & 38.6$\pm$1.9  & 0.6$\pm$0.1  &      18.7$\pm$1.1       \\ \midrule
	          FB15k           &        8.4$\pm$0.2        & 4.8$\pm$0.2 & 11.1$\pm$0.4 & 127.1$\pm$5.2 & 10.9$\pm$0.3 &      38.7$\pm$2.9       \\ \midrule
	        FB15k237          &        2.6$\pm$0.1        & 3.3$\pm$0.3 & 4.8$\pm$0.2  & 61.1$\pm$2.8  & 4.6$\pm$0.2  &      21.3$\pm$1.8       \\ \bottomrule
\end{tabular}
\end{table}

\subsection{Case Study}
Here, 
we demonstrate the number of operations of searched core tensor in the below Figure~\ref{figs:case}.
It indicates that S2S is data-specific, which can search various sparse core tensor $\mathcal{Z}^n$ for different data sets.

\begin{figure}[h]
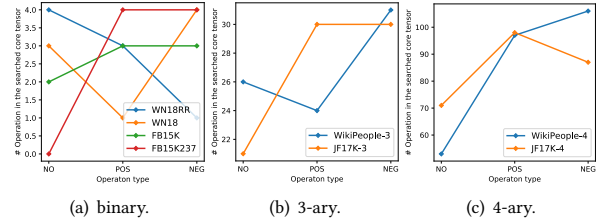

	\centering
	\subfigure[binary.]
	{\includegraphics[width=0.3\linewidth]{Figs/exp/case/case_binary.png}}
	\subfigure[3-ary.]
	{\includegraphics[width=0.295\linewidth]{Figs/exp/case/case_3_ary.png}}
	\subfigure[4-ary.]
	{\includegraphics[width=0.301\linewidth]{Figs/exp/case/case_4_ary.png}}
	\caption{The number of operations searched by S2S in several data sets. Note that NO, POS, NEG represents $\mathcal{I}_{0}^n$, $\mathcal{I}_{1}^n$, and $-\mathcal{I}_{1}^n$ respectively.}
	\label{figs:case}
\end{figure}

\subsection{Ablation Study}

\subsubsection{The Influence of the Joint Learning}
\label{sssec:ablationJoint}

As discussed in Section~\ref{sec:intro}, 
the tensor decomposition models only learn embedding from part of N-ary relational data, which causes the data sparsity issue to become more severe.
To verify this claim, we include another S2S (mixed) learned from N-ary relational data with mixed arity to compare the S2S (fixed) reported in Table~\ref{table:linkPredictionNary1} and Table~\ref{table:linkPredictionNary2}, which is learned from fixed arity.
It is obvious that S2S (mixed) achieves better performance, which demonstrates that only leveraging part of N-ary relational data indeed suffers from the data-sparsity issue.
This verifies that we need to propose a tensor decomposition model for the N-ary relational data learning.
We further discuss the effectiveness of proposed embedding sharing in Section~\ref{sssec:ablationEmbedShare}.

\begin{table}[h]
	\centering
	\caption{The performance comparison of S2S between fixed learning and joint learning.}
	\label{table:ablationSparsiyt}
	\begin{tabular}{c|cc|cc}
		\toprule
		\multirow{2}{*}{data set}         &      \multicolumn{2}{c|}{S2S (fixed)}  &      \multicolumn{2}{c}{S2S (mixed)}    \\
		&        MRR     &      Hits@10   &  MRR &   Hits@10         \\ \midrule
		WikiPeople-3 & {0.386} & 0.559 & {0.408}  & {0.577}  \\ 
		WikiPeople-4 & {0.391} & 0.600 & {0.418}& {0.617} \\ 
		JF17K-3 & {0.740} & {0.860} & {0.752}  & {0.870}  \\ 
		JF17K-4 & {0.822} & {0.924} & {0.831} & {0.934} \\ \bottomrule
	\end{tabular}
\end{table}

\subsubsection{The Influence of the Embedding Sharing Way}
\label{sssec:ablationEmbedShare}
In Section~\ref{sssec:ablationJoint}, we show that the sparsity issue exists when models only leverage part of N-ary relational data.
As discussed in Section~\ref{ssec:embedShare}, it is hard for tensor decomposition models to handle the N-ary relational data with mixed arity.
Directly sharing all embeddings across arities is too restrictive and lead to poor performance \cite{wen2016representation,zhang2018scalable,guan2019link}.
Therefore,
we propose to share embeddings based on segments.
To verify claims and investigate the influence of 
embedding sharing ways, 
we demonstrate the performance of 
several tensor decomposition models on WikiPeople and JF17K as in Figure~\ref{figs:exp:embedShare}.
Appendix~\ref{appendix:embedShare} introduces
the details of implementing embedding sharing into tensor decomposition models.

First, we can observe that
all tensor decomposition models achieve better performance with sharing embedding segments.
That is because embedding sharing not only makes the embedding learn from the low-arity fact in the high-order training but also maintain a part of high-order knowledge.
Second, it is clear that GTED and S2S achieve better performance than n-CP in N-ary relational data.
Unlike n-CP, GTED and S2S need to learn a core tensor for facts with every arity $n$.
The core tensor can encode the arity-specific knowledge, that further enhance the performance in joint learning.

\begin{figure}[!t]
	\centering
	\subfigure[WikiPeople.]
	{\includegraphics[width=0.495\linewidth]{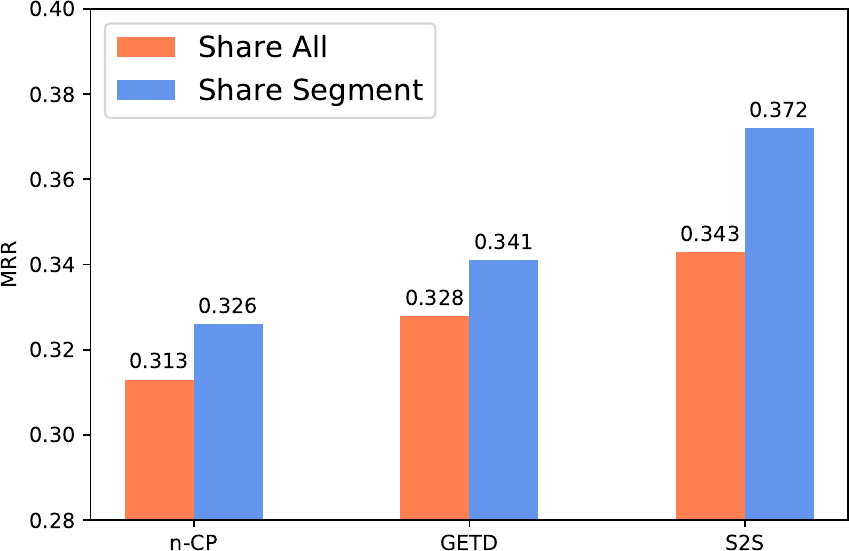}}
	\subfigure[JF17K.]
	{\includegraphics[width=0.495\linewidth]{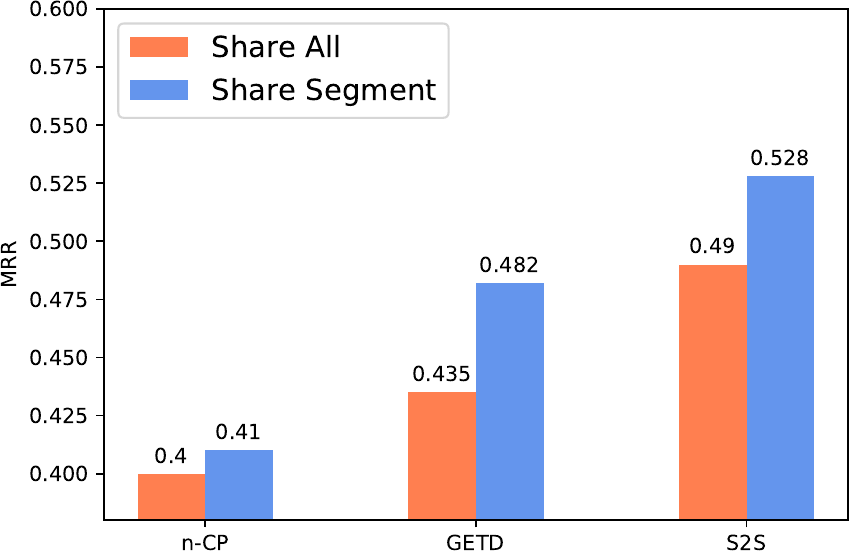}}
	\caption{The influence of different embedding sharing in tensor models.}
	\label{figs:exp:embedShare}
\end{figure}

\subsubsection{The Influence of the Model Complexity}
\label{sssec:ablationComplexity}
Previously, we discuss the negative effect of the over-parameterized issue in existing tensor decomposition models.
As mentioned in Section~\ref{sec:intro}, cubic or even larger model complexity is easy to make the model difficult to train.
Therefore, we here investigate the influence of model parameter size in 
Figure~\ref{figs:exp:modelPara}.
Note that we do not include the embedding as the model parameter since every model at least require $O(n_ed_e+n_rd_r)$ for embedding.
Thus we plot n-CP \cite{lacroix2018canonical} as a horizontal line since it has no extra parameter.
We can observe that S2S can achieve outstanding performance by requiring a small number of parameters.
And its performance does not vary greatly with the increase of model parameters.
On the contrary, GETD and n-TuckER require much larger parameter size to achieve the high performance.
And their model parameter setting will lead to significant differences in performance.
This may bring a difficulty to the training in practical, such as the careful selection of the size of model parameters.

\begin{figure}[ht]
	\centering
	\subfigure[WikiPeople-3.]
	{\includegraphics[width=0.46\linewidth]{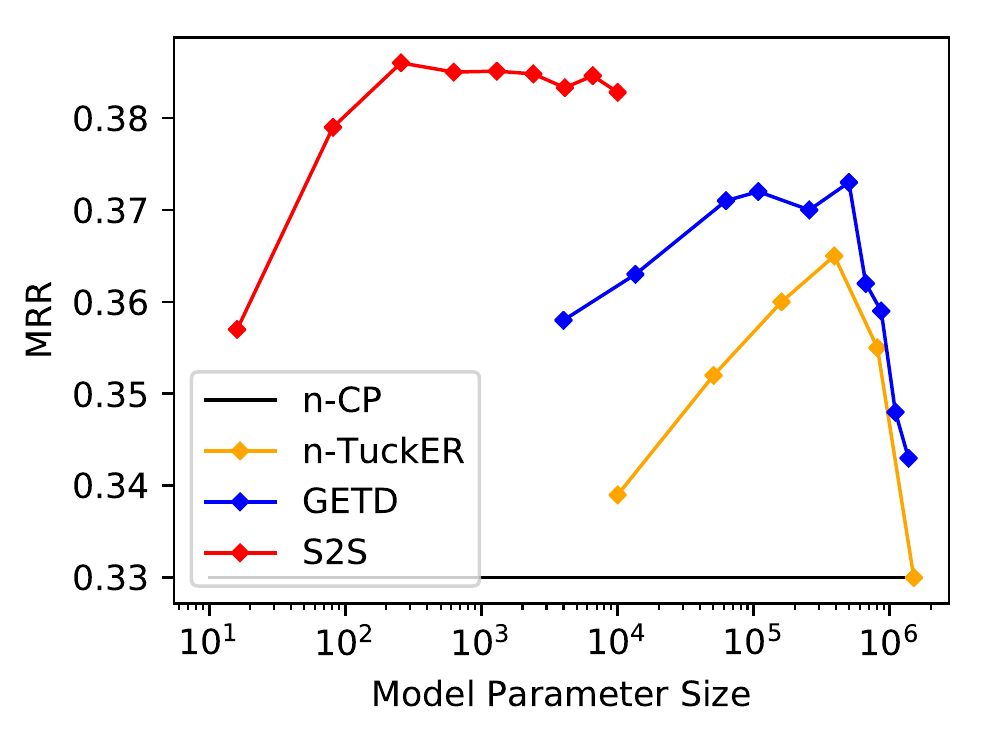}}
	\subfigure[JF17K-3.]
	{\includegraphics[width=0.45\linewidth]{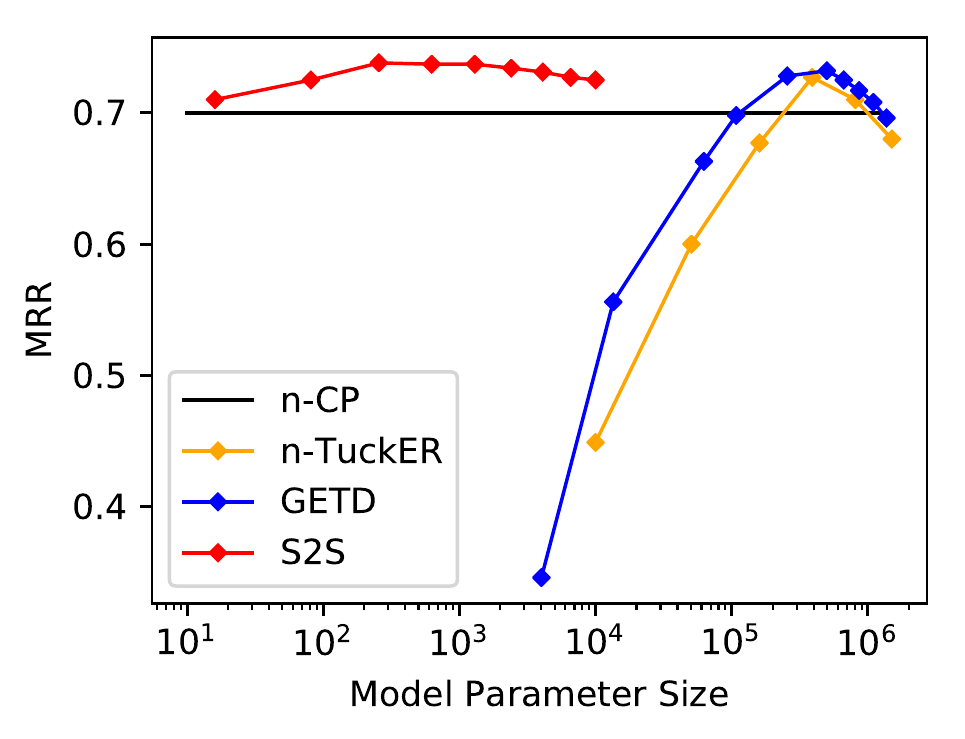}}
	\caption{The influence of model parameters.}
	\label{figs:exp:modelPara}
\end{figure}

\subsubsection{The Influence of the Structured Sparse Core Tensor}
\label{sssec:ablationL0}
We demonstrate the over-parameterization issue in Section~\ref{sssec:ablationComplexity}.
And we can observe that S2S achieves outstanding performance in 
Table 3-6.
To investigate the effectiveness of the proposed structured sparse core tensor,
we compare S2S with S2S(L0-reg),
which encourages the sparse core tensor by $\ell_0$ constraint.
S2S(L0-reg) has the same number of non-zero elements as S2S,
its sparse pattern is not structured and non-zero elements can be arbitrarily distributed across the core tensor.
Results are in Table~\ref{table:ablationRandom}.
We can observe that the performance of S2S(L0-reg) is much worse than the performance of S2S as reported in Table~\ref{table:linkPredictionNary1}-\ref{table:linkPredictionNary2}.
That is because the unstructured sparse core tensor cannot capture the correlation between embeddings as well as the structured one.
The implementation details have been introduced in Appendix~\ref{appendix:L0}.


\begin{table}[!t]
	\centering
	\caption{The link prediction performance of S2S(L0-reg).}
	\label{table:ablationRandom}
	\begin{tabular}{c|cc|cc}
		\toprule
		\multirow{2}{*}{data set}         &      \multicolumn{2}{c|}{S2S(L0-reg)}  &      \multicolumn{2}{c}{S2S}    \\
		&        MRR     &   Hits@10       & MRR & Hits@10  \\ \midrule
		WikiPeople-3 & 0.289 & 0.426  &0.386 & 0.559\\ 
		WikiPeople-4 & 0.288 & 0.457  & 0.391 & 0.600\\ 
		JF17K-3 & 0.665 & 0.774  & 0.740 & 0.860 \\ 
		JF17K-4 & 0.755 & 0.822 & 0.822 &0.924 \\ \bottomrule
	\end{tabular}
\end{table}

\subsubsection{The Impact of the Number of Segments}
We here investigate the effect of the different number of segments (i.e., $M$) on the N-ary relational data learning with fixed arity in Figure~\ref{figs:exp:segment}.
We can observe that S2S has good performance when the number of segments is set to 4. And the effect is not sensitive to the parameter setting.

\begin{figure}[ht]
	\centering
	\subfigure[WikiPeople-4.]
	{\includegraphics[width=0.45\linewidth]{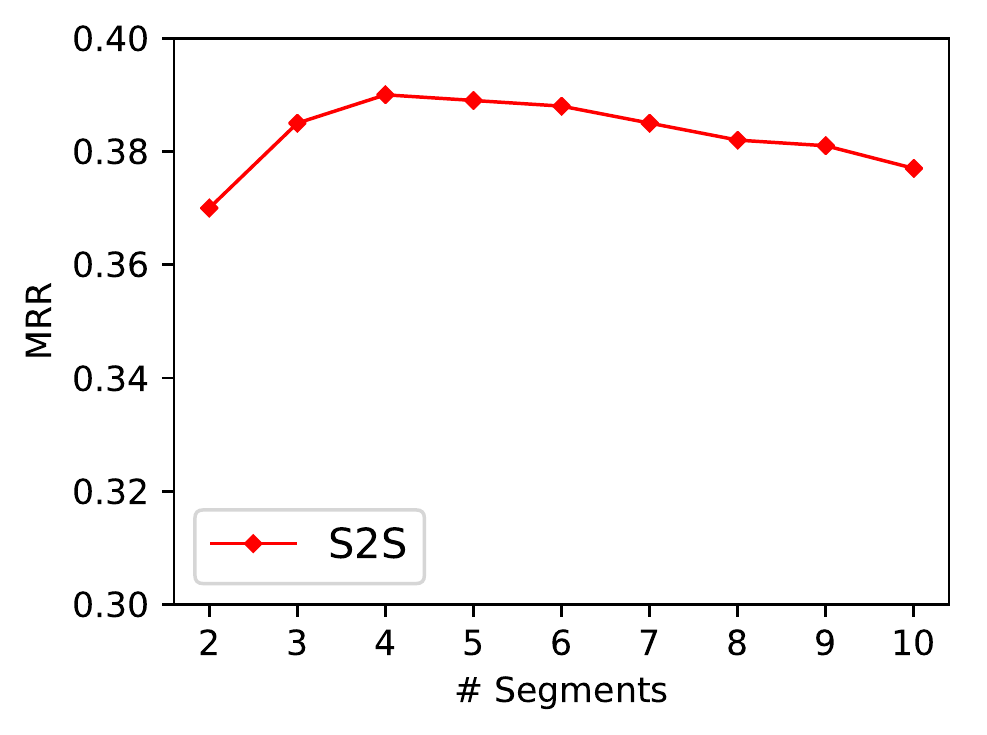}}
	\subfigure[JF17K-4.]
	{\includegraphics[width=0.45\linewidth]{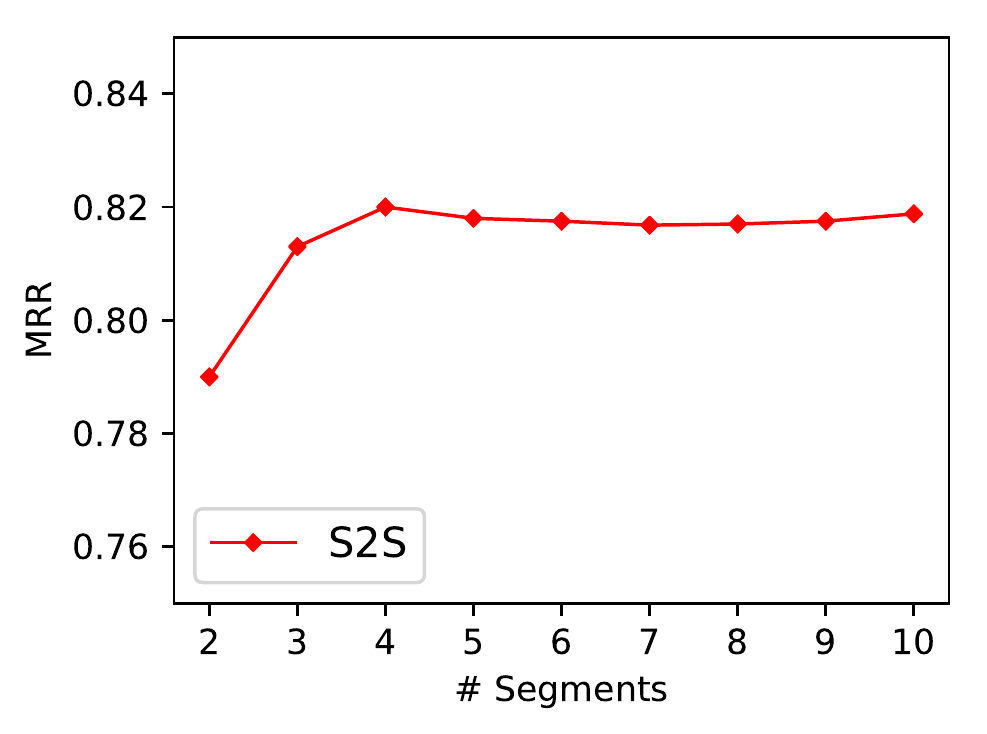}}
	\caption{The effects of the number of segments in S2S.}
	\label{figs:exp:segment}
\end{figure}

\subsubsection{Single v.s. Bi-level Formulation}
We follow NAS to formulate Definition~\ref{def:problem} into a bi-level optimization problem.
To investigate the impact of optimization level, we add a variant of S2S named S2S(sig), which optimizes \eqref{eq:problem} based on training data $S_{\text{tra}}$.
As shown in Table~\ref{table:ablationLevel}, the effect of S2S(sig) is generally lower than S2S.
That is because using validation data to optimize $\bm{\theta}$ will encourage the model to find core tensors that generalize well, rather than fitting the training data well.

\begin{table}[H]
\centering
\caption{The link prediction performance of S2S(sig).}
\label{table:ablationLevel}
\begin{tabular}{c|cc|cc}
	\toprule
	\multirow{2}{*}{data set}         &      \multicolumn{2}{c|}{S2S(sig)}  &      \multicolumn{2}{c}{S2S}    \\
	&        MRR     &   Hits@10       & MRR & Hits@10  \\ \midrule
	WikiPeople-3 & 0.377 & 0.545  &0.386 & 0.559\\ 
	WikiPeople-4 & 0.380 & 0.592  & 0.391 & 0.600\\ 
	JF17K-3 & 0.727 & 0.839  & 0.740 & 0.860 \\ 
	JF17K-4 & 0.800 & 0.908 & 0.822 &0.924 \\ \bottomrule
\end{tabular}
\end{table}


\section{Conclusion}

In this paper, we propose a new tensor decomposition model,
i.e., 
S2S, 
to learn embedding from the N-ary relational data. 
First, to alleviate the data-sparsity issue, we propose to segment embeddings into multiple parts and share them across arities by different segments.
Then, 
the proposed tensor decomposition model is able to learn from the N-ary relational data with mixed arity. 
Next, 
we present a new sparsifying method to address the over-parameterization issue in existing tensor decomposition models but maintain the expressiveness.
Experimental results on benchmark data sets demonstrate the effectiveness and efficiency of our proposed model S2S.

For future works, one interesting direction is to incorporate the N-ary relational data into kinds of applications.
For example, \cite{cao2019unifying} applies the link prediction task on KGs to the recommendation system.
However, it only leverages the binary relational data, which is a special form of N-ary relational data.
Since this paper provides a light way to handle the N-ary relational data,
we may be able to leverage the web-scale KBs to improve the performance of those applications.
Another direction worth trying is to model the N-ary relational data with multi-relational hypergraphs  and apply graph neural networks~\cite{yadati2020neural}.
It could be a more natural way to model the web-scale KBs instead of multiple tensors.

\section{Acknowledgements}
This work is partially supported by National Key Research and Development  Program of China Grant no. 2018AAA0101100, the Hong Kong RGC GRF Project 16202218 , CRF Project C6030-18G, C1031-18G, C5026-18G, AOE Project AoE/E-603/18, China NSFC No. 61729201, Guangdong Basic and Applied Basic Research Foundation 2019B151530001, Hong Kong ITC ITF grants ITS/044/18FX and ITS/470/18FX, Microsoft Research Asia Collaborative Research Grant, Didi-HKUST joint research lab project, and Wechat and Webank Research Grants.



\bibliographystyle{ACM-Reference-Format}
\bibliography{sample-base}


\cleardoublepage
\appendix
\section{Proof of Theorem~\ref{theorem1}}
\label{appendix:theory}
We first introduce two lemmas that will be used in the proof of Theorem~\ref{theorem1}.

\begin{lemma} \label{lemma1}
Given any N-ary relational data $S$ on the entity set $E$ and relation set $R$, n-CP \cite{lacroix2018canonical} can accurately represents the ground truth with $|S|$-dimensional embeddings, such as $\bm{E}, \bm{R} \in \mathbb{R}^{|S|}$. 
\end{lemma}

\begin{proof}
For any $k$-th fact in N-ary relational data $S$, such that $s = (r_{i_r}, e_{i_1}, \dots, e_{i_n})$.
Let the $k$-th element of $\bm{r}_{i_r}, \bm{e}_{i_1}, \dots, \bm{e}_{i_n}$ be 1, and set the $k$-th element of other $\bm{r}\in \bm{R}$ and $\bm{e}\in \bm{E}$ not involved in $s$ to 0.
Then, n-CP \cite{lacroix2018canonical} can accurately predict the given fact $s= (r_{i_r}, e_{i_1}, \dots, e_{i_n})$ is plausible if and only if $\langle \bm{r}_{i_r}, \bm{e}_{i_1}, \dots, \bm{e}_{i_n} \rangle \geq 1$, otherwise the fact is not fake.
If $\langle \bm{r}_{i_r}, \bm{e}_{i_1}, \dots, \bm{e}_{i_n} \rangle \geq 1$, there must at least have one dimension $k$ leads to $[\bm{r}_{i_r}]_{k} = [\bm{e}_{i_1}]_{k} = \dots = [\bm{e}_{i_n}]_{k} = 1$. Therefore, the given fact $(r_{i_r}, e_{i_1}, \dots, e_{i_n})$ is the $k$-th fact in the data $S$.
Similarly, if $(r_{i_r}, e_{i_1}, \dots, e_{i_n})$ exists, there must have $\langle \bm{r}_{i_r}, \bm{e}_{i_1}, \dots, \bm{e}_{i_n} \rangle \geq 1$.
\end{proof}

\begin{lemma}\label{lemma2}
The n-CP \cite{lacroix2018canonical} can be viewed as a special case of the S2S sparse core tensor.
\end{lemma}

\begin{proof}

Given the embedding $\bm{H} = \{\bm{E} \in \mathbb{R}^{n_e\times d}, \bm{R}\in \mathbb{R}^{n_e\times d}\}$, we first segment the embedding into $m$ parts, such as $\bm{e}_{i} = [\bm{e}_{i}^{(1)}; \cdots; \bm{e}_{i}^{(m)}]$.
Then, n-CP's \cite{lacroix2018canonical} scoring function to measure $s=(r_{i_r},e_{i_1},\dots,e_{i_n})$ is defined as: 	
\begin{align}\label{eq:11}
\!
f(s, \bm{H}) 
\!=\! 
\left\langle 
\bm{r}_{i_r}, 
\bm{e}_{i_1},
\dots,
\bm{e}_{i_n}
\right\rangle 
\!=\! 
\sum_{j=1}^m\!
\left\langle 
\bm{r}_{i_r}^{(j)}, 
\bm{e}_{i_1}^{(j)},
\dots,
\bm{e}_{i_n}^{(j)}
\right\rangle.
\end{align}
Next we prove that \eqref{eq:11} is a special case of S2S's scoring function, which is initially defined with a sparse core tensor $\mathcal{Z}^n = \{\mathcal{Z}^n_k\}_{k=1}^K$ as:
\begin{equation}
\label{eq:12}
f_z(\bm{H}, s; \mathcal{Z}^n)
\!=
\!\!\!\!\!\!\!\sum_{j_r,j_1,\dots,j_{n}}
\!\!\!\!
\mathcal{Z}_k^n 
\!\times_1\! \bm{r}_{i_r}^{j_r}
\!\times_2\! \bm e_{i_1}^{j_1} 
\!\times_3\! \cdots 
\!\times_{n+1}\! \bm{e}_{i_n}^{j_n},
\end{equation}
where $j_r,j_1,\dots,j_r\in\{1,\dots,m\}$ and $\mathcal{Z}_k^n\in \text{\tt OP}=\{\mathcal{I}_{-1}^n, \mathcal{I}_0^n, \mathcal{I}_1^n\}$.
Because $\mathcal{I}_{v}$ is super-diagonal with $v$, 
\eqref{eq:12} actually perform the tensor computation as follows:
\begin{align*}
f_z(\bm{H}, s; \mathcal{Z}^n)
&=\!\sum_{j_r,j_1,\dots,j_{n}}\! \mathcal{Z}_k^n \times_1 \bm{r}_{i_r}^{j_r}\times_2 \bm e_{i_1}^{j_1} \times_3 \cdots \times_{n+1} \bm{e}_{i_n}^{j_n},\\
& = \sum_{j_r,j_1,\dots,j_{n}} 
v
\cdot 
\left\langle 
\bm{r}_{i_r}^{(j_r)}, \bm{e}_{i_1}^{j_1},\dots,\bm{e}_{i_n}^{j_n}
\right\rangle.
\end{align*}
Then, we let $v=1$ if and only if $j_r=j_1=\dots=j_n$. The above equation will converted to:
\begin{equation*}
f_z(\bm{H}, s; \mathcal{Z}^n)
 = \!\!\!\! \sum_{j_r=j_1\dots=j_{n}=1}^m  \!\!\!\!
1
\cdot 
\left\langle 
\bm{r}_{i_r}^{(j_r)}, \bm{e}_{i_1}^{j_1},\dots,\bm{e}_{i_n}^{j_n}
\right\rangle,
\end{equation*}
that is exactly same with $f(s,\bm{H})$ in \eqref{eq:11}.
Therefore, n-CP \cite{lacroix2018canonical} is actually a special case of S2S.
\end{proof}

According to Lemma~\ref{lemma1}, n-CP \cite{lacroix2018canonical} is expressive enough to handle any N-ary relational data, and n-CP is a special case of S2S as shown in Lemma~\ref{lemma2}.
Therefore, S2S has the sparse core tensor to represent the ground truth of any N-ary relational data.
%
%

\section{Experimental Implementation}
\subsection{Embedding Sharing in Other Tensor Decomposition Models}
\label{appendix:embedShare}

In Section~\ref{sssec:ablationEmbedShare}, we implement the embedding sharing idea mentioned in Sec~\ref{ssec:embedShare} into other tensor decomposition models.
Here we briefly introduce the exact implementation.

Same with S2S, given the maximum arity $N$ and number of segments $M$, we segment embeddings into 
$M$ splits, 
i.e., $\bm{e}_i = [\bm{e}^{1}_i; \dots; \bm{e}^{M}_i ]$.
Then, given a fact $s$ with arity $n$, we utilize first $m$-th (i.e., $m=\min\{n,M\}$) segments of embeddings to compute the score in DistMult~\cite{yang2014embedding} and GETD~\cite{liu2020generalizing}.
The corresponding DistMult's scoring functions is defined as:
\begin{equation*}
f(s, \bm{H}) 
=
\sum_{j=1}^{m}
\left\langle 
\bm{r}_{i_r}^j, 
\bm{e}_{i_1}^j,
\dots,
\bm{e}_{i_n}^j
\right\rangle.
\end{equation*}
Moreover, 
the TuckER's scoring function is defined as:
\begin{align*}
f(s, \bm{H})
&=
\mathcal{G}^{n}
\times_1 \bm{r}_{i_r}^{1:m}
\times_2 \bm e_{i_1}^{1:m} 
\times_3 \cdots \times_{n+1} \bm{e}_{i_{n}}^{1:m}\\
&\approx
\text{TR}\left(\mathcal{W}_1,\cdots,\mathcal{W}_c\right)
\times_1 \bm{r}_{i_r}^{1:m}
\times_2 \bm e_{i_1}^{1:m} 
\times_3 \cdots \times_{n+1} \bm{e}_{i_{n}}^{1:m}
\end{align*}
where $\bm{r}_{i_r}^{1:m}, \bm e_{i}^{1:m}$ represent the vector with first $m$-th segments (e.g., $\bm e_{i}^{1:m} = [\bm{e}_i^1;\dots;\bm{e}_i^m] $), and
$\mathcal{G}^{n}$ is a $n+1$-order Tucker core tensor with size $\nicefrac{md}{M}$ (e.g., $\mathcal{G}^{2}\in \mathbb{R}^{\nicefrac{2d}{M}\times\nicefrac{2d}{M}\times\nicefrac{2d}{M}}$).
Then, $\text{TR}\left( \cdot \right)$ is achieved by Tensor Ring computation~\cite{zhao2016tensor} as mentioned in Section~\ref{ssec:relatedNary}.

\subsection{Sparsify Core Tensor with L0 Constraint}
\label{appendix:L0}
Here we introduce 
the details of S2S(L0-reg), i.e.,
how to sparsify the core tensor with $\ell_0$ constraint as in Section~\ref{sssec:ablationL0}.
To optimize the core
We first give the optimization objective as:
\begin{equation*}
\arg\min_{\mathcal{Z}} 
L(\bm{H}, \mathcal{Z}; S_{\text{val}})
+
\epsilon
\left\Vert
\mathcal{Z}
\right\Vert_0,
\end{equation*}
where $\epsilon$ is a trade-off weight for the multi-class log loss $L$ and regularization.
The $\ell_0$ norm penalizes the number of non-zero entries in the core tensor $\mathcal{Z}=\{\mathcal{Z}^n\}_{n=2}^N$ (e.g., $\mathcal{Z}^n_{j_r,j_1,\dots,j_n}\neq0$).
Note that S2S(L0-reg) has the same number of non-zero elements as S2S, i.e., $m^{n+1}$ for $\mathcal{Z}^n$.
Optimizing above objective 
is 
computationally intractable
because of its non-differentiability and the exponential complexity.
To minimize the objective,
we adopt the technique proposed in \cite{louizos2017learning},
which utilizes the reparameterization trick to make it differentiable.

\end{document}